\documentclass{article}

\usepackage{iclr2020_conference,times}

\usepackage[utf8x]{inputenc}                         
\usepackage[T1]{fontenc}    
\usepackage{hyperref}     

\usepackage{amsmath, amssymb, amsthm}
\usepackage[capitalize, noabbrev]{cleveref}

\usepackage{url}          
\usepackage{booktabs}       
\usepackage{amsfonts}    
\usepackage{nicefrac}    
\usepackage{microtype}   

\usepackage{bbold}

\usepackage{wrapfig}

\usepackage{subcaption}

\usepackage{multicol}

\usepackage{pgfplots} 
\usepackage{pgfplotstable} 
\pgfplotsset{compat=1.13}
\usepgfplotslibrary{colorbrewer}
\usepgfplotslibrary{fillbetween}

\theoremstyle{definition} 
\newtheorem{theorem}{Theorem}[section]
\theoremstyle{definition}
\newtheorem{lemma}[theorem]{Lemma}
\theoremstyle{definition}
\newtheorem{definition}[theorem]{Definition}

\usepackage{xcolor} 

\definecolor{darkred}{rgb}{0.5,0,0}
\definecolor{darkgreen}{rgb}{0,0.5,0}
\definecolor{darkblue}{rgb}{0,0,0.5}
\definecolor{lightgrey}{rgb}{0.7,0.7,0.7}
\definecolor{lightergrey}{rgb}{0.93,0.93,0.93}

\usepackage{acro} 
\DeclareAcronym{cli} {
    short = CLI,
    long = Command Line Interface,
    class = abbrev
}

\usepackage{titlecaps}
\Resetlcwords
\Addlcwords{a after along an and are around as at but by etc}
\Addlcwords{for from if in is nor of on the so to with without yet}
\Addlcwords{}

\newcommand{\eg}{e.g., }

\usepackage{lipsum}

\ifdefined\isoverfull
\else
\fi

\usepackage{tikz}
\usetikzlibrary{calc,decorations.pathmorphing,shapes}
\usetikzlibrary{positioning,shapes,fit,arrows}
\usetikzlibrary{decorations.markings}
\usetikzlibrary{shapes,shapes.geometric}
\usetikzlibrary{shadows,patterns}
\usetikzlibrary{backgrounds,decorations.pathreplacing,automata}

\newtheorem*{theorem*}{Theorem}

\iclrfinalcopy

\title{Universal Approximation \\with Certified Networks}

\author{Maximilian Baader, Matthew Mirman, Martin Vechev \\
Department of Computer Science\\
ETH Zurich, Switzerland\\
\{\texttt{mbaader,matthew.mirman,martin.vechev}\}\texttt{@inf.ethz.ch}}

\begin{document}

\maketitle

\begin{abstract}
    Training neural networks to be certifiably robust is critical to ensure their safety against adversarial attacks. However, it is currently very difficult to train a neural network that is both accurate and certifiably robust. In this work we take a step towards addressing this challenge. We prove that for every continuous function $f$, there exists a network $n$ such that:
(i) $n$ approximates $f$ arbitrarily close, \emph{and} (ii) simple interval bound propagation
of a region $B$ through $n$ yields a result that is arbitrarily close to the optimal output of $f$ on $B$.
Our result can be seen as a Universal Approximation Theorem for interval-certified ReLU networks.
To the best of our knowledge, this is the first work to prove the existence of accurate, interval-certified networks.

\end{abstract}

\section{Introduction}

Much recent work has shown that neural networks can be fooled into
misclassifying adversarial examples \citep{adversarialDiscovery}, inputs which
are imperceptibly different from those that the neural network classifies correctly.
Initial work on defending against adversarial examples revolved around training
networks to be empirically robust, usually by including adversarial examples
found with various attacks into the training dataset \citep{gu2014towards,
papernot2016limitations, zheng2016improving, athalye2017synthesizing,
evtimov2017robust, moosavi2017universal, xiao2018generating}. However, while
empirical robustness can be practically useful, it does not provide safety
guarantees. As a result, much recent research has focused on verifying that a
network is certifiably robust, typically by employing methods based on mixed
integer linear programming \citep{tjeng2017evaluating}, SMT solvers
\citep{katz2017reluplex}, semidefinite programming \citep{RaghunathanSL18a},
duality \citep{kolter2018provable, krishnamurthy2018dual}, and linear relaxations
\citep{ai2, FastLin2018, WangPWYJ18, CrownIBP, eran, Salman}.

Because the certification rates were far from satisfactory, specific training
methods were recently developed which produce networks that are
certifiably robust: \citet{diffai, RaghunathanSL18b, mixtrain,
kolter2018provable, wong2018scaling, ibp} train the network with standard
optimization applied to an over-approximation of the network behavior on a given
input region (the region is created around the concrete input point). These
techniques aim to discover specific weights which facilitate verification.
There is a tradeoff between the degree of the over-approximation used and the
speed of training and certification.
Recently, \citep{CohenRK19} proposed a statistical approach to certification, which unlike the non-probabilistic methods
discussed above, creates a probabilistic classifier that comes with probabilistic guarantees.

So far, some of the best non-probabilistic results achieved on the popular MNIST
\citep{mnist} and CIFAR10 \citep{cifar} datasets have been obtained with the
simple Interval relaxation \citep{ibp, diffai2}, which scales well at both
training and verification time.
Despite this progress, there are still substantial gaps between known standard
accuracy, experimental robustness, and certified robustness. For example, for
CIFAR10, the best reported certified robustness is 32.04\% with an accuracy of
49.49\% when using a fairly modest $l_\infty$ region with radius 8/255
\citep{ibp}. The state-of-the-art non-robust accuracy for this dataset is $>$
95\% with experimental robustness $>$ 50\%. Given the size of this gap, a
key question then is: \emph{can certified training ever succeed or is there a
fundamental limit}?

\begin{wrapfigure}{r}{0.33\textwidth}
    \centering
    \includegraphics[width=0.31\textwidth]{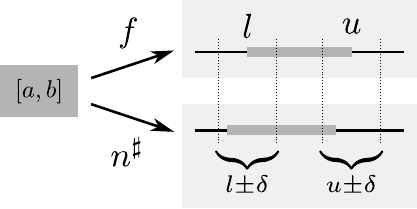}
    \caption{Illustration of \\ \cref{thm:intro:universal_approx_thm}.}
    \label{fig:thm_expl}
    \vspace{-0.5cm}
\end{wrapfigure}

In this paper we take a step in answering this question by proving a result
parallel to the Universal Approximation Theorem \citep{cybenko1989approximation,
hornik1989multilayer}. We prove that for any continuous function $f$ defined on
a compact domain $\Gamma \subseteq \mathbb{R}^m$ and for any desired level of
accuracy $\delta$, there exists a ReLU neural network $n$ which can certifiably
approximate $f$ up to $\delta$ using interval bound propagation. As an interval
is a fairly imprecise relaxation, our result directly applies to more precise
convex relaxations (\eg \citet{CrownIBP, DeepPoly}).

\begin{theorem}[Universal Interval-Certified Approximation, \cref{fig:thm_expl}]
    \label{thm:intro:universal_approx_thm}
    Let $\Gamma \subset \mathbb{R}^m$ be a compact set and let $f \colon \Gamma
    \to \mathbb{R}$ be a continuous function. For all $\delta > 0$, there exists
    a ReLU network $n$ such that for all boxes $[a,b]$ in $\Gamma$ defined by
    points $a, b \in \Gamma$ where $a_k \leq b_k$ for all $k$, the propagation
    of the box $[a, b]$ using interval analysis through the network $n$, denoted
    $n^\sharp([a, b])$, approximates the set $[l, u] = [\min f([a, b]), \max
    f([a, b])] \subseteq \mathbb{R}$ up to $\delta$,
        \begin{equation}
            [l + \delta, u - \delta] \subseteq
            n^\sharp([a, b]) \subseteq
            [l - \delta, u + \delta]
            .
            \label{eq:introduction}
        \end{equation}
\end{theorem}

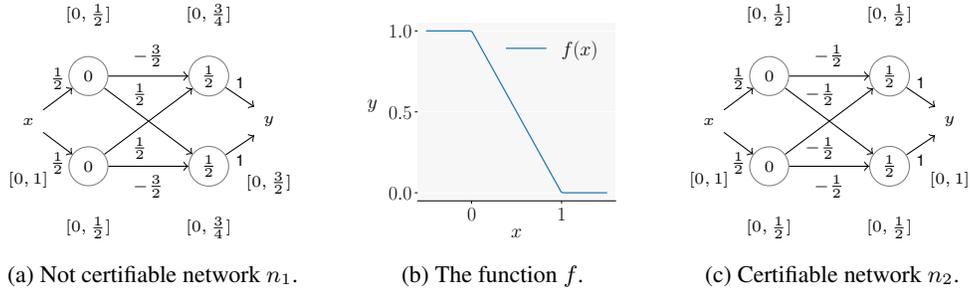
\begin{figure*}[t]
    \centering
    \hspace{-1.2cm}
    \begin{subfigure}[t]{.4\textwidth}
        \centering
        \begin{tikzpicture}[scale=0.8, font=\tiny\sffamily]
    \pgfmathsetmacro{\l}{1.5}
    \pgfmathsetmacro{\w}{2}

    \node[] (X) at (- 0.5 * \w, - 0.5 * \l) {$x$}; 

    \node[shape=circle,draw=gray] (A) at (0,0) {$0$};
    \node[shape=circle,draw=gray] (B) at (0,-\l) {$0$};

    \node[shape=circle,draw=gray, inner sep=1.8pt] (C) at (\w,0) {$\tfrac{1}{2}$};
    \node[shape=circle,draw=gray, inner sep=1.8pt] (D) at (\w,-\l) {$\tfrac{1}{2}$};

    \node[] (FX) at (\w + 0.5 * \w, - 0.5 * \l) {$y$};

    \path [->] (X) edge node[left,above] {$\tfrac{1}{2}$} (A);    ​
    \path [->] (X) edge node[left,below] {$\tfrac{1}{2}$} (B);

    \path [->] (A) edge node[left,above] {$-\tfrac{3}{2}$} (C);
    \path [->] (B) edge node[left,below] {$-\tfrac{3}{2}$} (D);
    \path [->] (A) edge node[pos=0.4,left,above] {$\tfrac{1}{2}$} (D);
    \path [->] (B) edge node[pos=0.4,left,below] {$\tfrac{1}{2}$} (C);
    
    \path [->] (C) edge node[right,above] {\tiny 1} (FX.north west);
    \path [->] (D) edge node[right,below] {\tiny 1} (FX.south west);

    \coordinate (x0) at (X);
    \node at (x0) [below=5.5mm of x0] { $[0,1]$ };

    \coordinate (x1) at (A);
    \node at (x1) [above=5.5mm  of x1] { $[0, \tfrac{1}{2}]$ };
​
    \coordinate (x2) at (B);
    \node at (x2) [below=5.5mm  of x2] { $[0, \tfrac{1}{2}]$ };
​​

    \coordinate (x3) at (C);
    \node at (x3) [above=5.5mm  of x3] { $[0, \tfrac{3}{4}]$ };
​
    \coordinate (x4) at (D);
    \node at (x4) [below=5.5mm  of x4] { $[0, \tfrac{3}{4}]$ };
​

    \coordinate (x5) at (FX);
    \node at (x5) [below=5.5mm of x5] { $[0,\tfrac{3}{2}]$ }; ​

\end{tikzpicture}
        \subcaption{Not certifiable network $n_1$.}
        \label{fig:network:notprovable}
    \end{subfigure}
    \hspace{-0.2cm}
    \begin{subfigure}[t]{.25\textwidth}
        \centering
        \includegraphics[width=\textwidth]{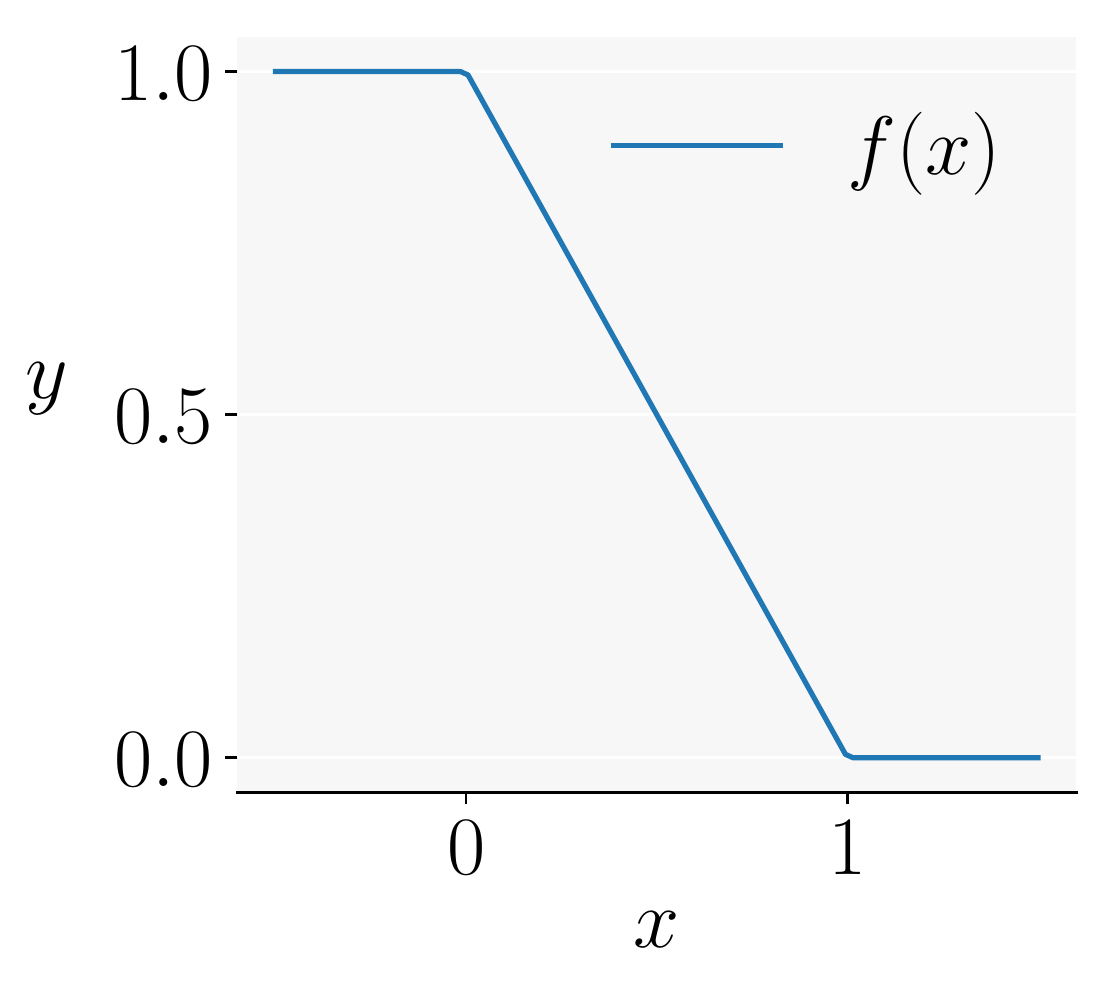}
        \subcaption{The function $f$.}
        \label{fig:network:f}
    \end{subfigure}
    \hspace{-0.2cm}
    \begin{subfigure}[t]{.4\textwidth}
        \centering
        \begin{tikzpicture}[scale=0.8, font=\tiny\sffamily]
    \pgfmathsetmacro{\l}{1.5}
    \pgfmathsetmacro{\w}{2}

    \node[] (X) at (- 0.5 * \w, - 0.5 * \l) {$x$}; 

    \node[shape=circle,draw=gray] (A) at (0,0) {$0$};
    \node[shape=circle,draw=gray] (B) at (0,-\l) {$0$};

    \node[shape=circle,draw=gray, inner sep=1.8pt] (C) at (\w,0) {$\tfrac{1}{2}$};
    \node[shape=circle,draw=gray, inner sep=1.8pt] (D) at (\w,-\l) {$\tfrac{1}{2}$};

    \node[] (FX) at (\w + 0.5 * \w, - 0.5 * \l) {$y$};

    \path [->] (X) edge node[left,above] {$\tfrac{1}{2}$} (A);    ​
    \path [->] (X) edge node[left,below] {$\tfrac{1}{2}$} (B);

    \path [->] (A) edge node[left,above] {$-\tfrac{1}{2}$} (C);
    \path [->] (B) edge node[left,below] {$-\tfrac{1}{2}$} (D);
    \path [->] (A) edge node[pos=0.4,left,above] {$-\tfrac{1}{2}$} (D);
    \path [->] (B) edge node[pos=0.4,left,below] {$-\tfrac{1}{2}$} (C);
    
    \path [->] (C) edge node[right,above] {\tiny 1} (FX.north west);
    \path [->] (D) edge node[right,below] {\tiny 1} (FX.south west);

    \coordinate (x0) at (X);
    \node at (x0) [below=5.5mm of x0] { $[0,1]$ };

    \coordinate (x1) at (A);
    \node at (x1) [above=5.5mm  of x1] { $[0, \tfrac{1}{2}]$ };
​
    \coordinate (x2) at (B);
    \node at (x2) [below=5.5mm  of x2] { $[0, \tfrac{1}{2}]$ };
​​

    \coordinate (x3) at (C);
    \node at (x3) [above=5.5mm  of x3] { $[0, \tfrac{1}{2}]$ };
​
    \coordinate (x4) at (D);
    \node at (x4) [below=5.5mm  of x4] { $[0, \tfrac{1}{2}]$ };
​

    \coordinate (x5) at (FX);
    \node at (x5) [below=5.5mm of x5] { $[0,1]$ }; ​

\end{tikzpicture}
        \subcaption{Certifiable network $n_2$.}
        \label{fig:network:provable}
    \end{subfigure}
    \hspace{-1.2cm}
    \caption{The ReLU networks $n_1$ (\cref{fig:network:notprovable}) and $n_2$
    (\cref{fig:network:provable}) encode the same function $f$
    (\cref{fig:network:f}). Interval analysis fails certify that $n_1$ does
    not exceed $[0,1]$ on $[0,1]$ while certification succeeds for $n_2$. }
    \label{fig:network}
\end{figure*}

We recover the classical universal approximation theorem ($|f(x) - n(x)| \leq
\delta$ for all $x \in \Gamma$) by considering boxes $[a, b]$ describing points
($x = a = b$). Note that here the lower bound is not $[l,u]$ as the network $n$ is an
approximation of $f$. Because interval analysis propagates boxes, the theorem naturally handles
$l_\infty$ norm bound perturbations to the input. Other $l_p$ norms can be
handled by covering the $l_p$ ball with boxes. The theorem can be extended
easily to functions $f \colon \Gamma \to \mathbb{R}^k$ by applying the theorem
component wise. 

\paragraph{Practical meaning of theorem}
The practical meaning of this theorem is as follows: if we train a neural network $n'$ on a given training data set (\eg CIFAR10)
and we are satisfied with the properties of $n'$ (\eg high accuracy), then because $n'$ is a continuous function,
the theorem tells us that there exists a network $n$ which is as accurate as $n'$ and as certifiable with interval
analysis as $n'$ is with a complete verifier. This means that if we fail to find such an $n$, then either $n$ did not possess the 
required capacity or the optimizer was unsuccessful.

\paragraph{Focus on the existence of a network}
We note that we do not provide a method for training a certified ReLU network --
even though our method is \emph{constructive}, we aim to answer an existential question and thus we focus on proving that a
given network exists. Interesting future work items would be to study the
requirements on the size of this network and the inherent hardness of finding
it with standard optimization methods.

\paragraph{Universal approximation is insufficient}
We now discuss why classical universal approximation is insufficient for
establishing our result.
While classical universal approximation theorems state that neural networks can
approximate a large class of functions $f$, unlike our result, they do not state
that robustness of the approximation $n$ of $f$ is actually certified with a
scalable proof method (\eg interval bound propagation).
If one uses a non scalable complete verifier instead, then the standard
Universal approximation theorem is sufficient.

To demonstrate this point, consider the function $f: \mathbb{R} \rightarrow
\mathbb{R}$ (\cref{fig:network:f}) mapping all $x \leq 0$ to $1$, all $x \geq 1$
to $0$ and all $0 < x < 1$ to $1-x$ and two ReLU networks $n_1$
(\cref{fig:network:notprovable}) and $n_2$ (\cref{fig:network:provable})
perfectly approximating $f$, that is $n_1(x) = f(x) = n_2(x)$ for all $x$.
For $\delta = \tfrac{1}{4}$, the interval certification that $n_1$ maps all $x
\in [0,1]$ to $[0,1]$ fails because $[\tfrac{1}{4}, \tfrac{3}{4}] \subseteq
n_1^\sharp([0,1]) = [0, \tfrac{3}{2}] \not\subseteq [-\tfrac{1}{4},
\tfrac{5}{4}]$. However, interval certification succeeds for $n_2$, because
$n_2^\sharp([0,1]) = [0,1]$.
To the best of our knowledge, this is the first work to prove the existence of
accurate, interval-certified networks.

\section{Related work}

After adversarial examples were discovered by \citet{adversarialDiscovery}, many
attacks and defenses were introduced (for a survey, see
\cite{akhtar2018threat}). Initial work on verifying neural network robustness
used exact methods \citep{katz2017reluplex, tjeng2017evaluating} on small
networks, while later research introduced methods based on over-approximation
\citep{ai2, RaghunathanSL18a, eran, Salman} aiming to scale to larger networks. A
fundamentally different approach is randomized smoothing
\citep{li2018certified,LecuyerAG0J19,CohenRK19}, in which probabilistic
classification and certification with high confidence is performed.

As neural networks that are experimentally robust need not be certifiably
robust, there has been significant recent research on training certifiably
robust neural networks \citep{RaghunathanSL18b, diffai, diffai2,
kolter2018provable, wong2018scaling, mixtrain, ibp, dvijotham2018training,
xiao2018training, CohenRK19}. As these methods appear to have reached a
performance wall, several works have started investigating the fundamental
barriers in the datasets and methods that preclude the learning of a robust
network (let alone a certifiably robust one) \citep{khoury2018geometry,
schmidt2018adversarially, tsipras2018robustness}. In our work, we focus on the
question of whether neural networks are capable of approximating functions whose
robustness can be established with the efficient interval relaxation.

\paragraph{Feasibility Results with Neural Networks}
Early versions of the Universal Approximation Theorem were stated by
\citet{cybenko1989approximation} and \citet{hornik1989multilayer}.
\citet{cybenko1989approximation} showed that networks using sigmoidal
activations could approximate continuous functions in the unit hypercube, while
\citet{hornik1989multilayer} showed that even networks with only one hidden
layer are capable of approximating Borel measurable functions.

More recent work has investigated the capabilities of ReLU networks.
Here, \citet{arora2018understanding}, based on \citet{tarela1999LatticePWL}, proved
that every continuous piecewise linear function in $\mathbb{R}^m$ can be
represented by a ReLU network. Later, \citet{ReLU_DNN_FiniteElements}
reduced the number of neurons needed using ideas from finite elements
methods.
Relevant to our work, \citet{arora2018understanding} introduced a ReLU network
representations of the $\min$ function. Further, we use a construction method
that is similar to the construction for nodal basis functions given in
\citet{ReLU_DNN_FiniteElements}.

Universal approximation for Lipschitz constrained networks have been considered
by \citet{anil2018sorting} and later by \citet{cohen2019universal}. 
A bound on the Lipschitz constant of a network immediately yields a certified region
depending on the classification margin. \citet{anil2018sorting} proved that the
set of Lipschitz networks with the GroupSort activation is dense in the space of
Lipschitz continuous functions with Lipschitz constant 1, while
\citet{cohen2019universal} provide an explicit construction to obtain the
network. We note that both of these works focus on Lipschitz continuous functions, a more restricted class
than continuous functions, which we consider in our work.

\section{Background} \label{sec:background}

In this section we provide the concepts necessary to describe our main result.

\paragraph{Adversarial Examples and Robustness Verification}

Let $n : \mathbb{R}^m \rightarrow \mathbb{R}^k$ be a neural network, which
classifies an input $x$ to a label $t$ if $n(x)_t > n(x)_j$ for all $j \neq t$.
For a correctly classified input $x$, an adversarial example is an input $y$
such that $x$ is imperceptible from $y$ to a human, but is classified to a
different label by $n$. 

Frequently, two images are assumed to be ``imperceptible'' if there $l_p$
distance is at most $\epsilon$. The $l_p$ ball around an image is said to be the
adversarial ball, and a network is said to be $\epsilon$-robust around $x$ if
every point in the adversarial ball around $x$ classifies the same. 
In this paper, we limit our discussion to $l_\infty$ adversarial balls which can
be used to cover to all $l_p$ balls. 

The goal of robustness verification is to show that for a neural network $n$,
input point $x$ and label $t$, every possible input in an $l_\infty$ ball of
size $\epsilon$ around $x$ (written $\mathbb{B}^\infty_{\epsilon}(x)$) is also
classified to $t$.

\paragraph{Verifying neural networks with Interval Analysis}
The verification technique we investigate in this work is interval analysis. We
denote by $\mathcal{B}$ the set of boxes $B = [a, b] \subset \mathbb{R}^m$ for all
$m$, where $a_i \leq b_i$ for all $i$. Furthermore for $\Gamma \subseteq
\mathbb{R}^m$ we define $\mathcal{B}(\Gamma) := \mathcal{B} \cap \Gamma$
describing all the boxes in $\Gamma$.
The standard interval-transformations for the basic operations we
are considering, namely $+, -, \cdot$ and the ReLU function $R$  (\citet{ai2}, \citet{ibp}) are

\begin{minipage}{.5\linewidth}
    \begin{align*}
        [a, b] +^\sharp [c, d] &= [a + c, b + d] \\
        R^\sharp ([a, b]) &= [R(a), R(b)]
    \end{align*}
\end{minipage}%
\begin{minipage}{.5\linewidth}
        \begin{align*}
            -^\sharp [a, b] &= [- b, - a] \\
            \lambda \cdot^\sharp [a, b] &= [\lambda a, \lambda b],
        \end{align*}
\end{minipage}
\vspace{0.3cm}

where $[a, b], [c, d] \in \mathcal{B}(\mathbb{R})$, and $\lambda \in
\mathbb{R}_{\geq 0}$. Furthermore, we used $\sharp$ to distinguish the function $f$
from its interval-transformation $f^\sharp$.
To illustrate the difference between $f$ and $f^\sharp$, consider $f(x) := x -
x$ evaluated on $x = [0, 1]$. We have $f([0,1]) = 0$, but $f^\sharp([0,1]) =
[0,1] -^\# [0, 1] = [0, 1] +^\# [-1, 0] = [-1, 1]$ illustrating the loss in
precision that interval analysis suffers from. 

Interval analysis provides a sound over-approximation in the sense that for all
function $f$, the values that $f$ can obtain on $[a,b]$, namely $f([a,b]) :=
\{f(x)\mid x \in [a,b] \}$ are a subset of $f^\sharp([a,b])$. If $f$ is a
composition of functions, $f = f_1 \circ \cdots \circ f_k$, then $f_1^\sharp
\circ \cdots \circ f_k^\sharp$ is a sound interval-transformer for $f$. 

Furthermore all combinations $f$ of $+,-,\cdot$ and $R$ are monotone, that is for
$[a,b], [c,d] \subseteq \mathcal{B}(\mathbb{R}^m)$ such that $[a, b] \subseteq
[c, d]$ then $f^\#([ a, b]) \subseteq f^\#([ c, d])$
(\cref{app:universal_approx_thm}). For boxes $[x, x]$ representing points
$f^\sharp$ coincides with $f$, $f^\sharp([x,x]) = f(x)$. This will later be
needed.

\begin{figure*}[t]
    \centering
    \begin{subfigure}[t]{.32\textwidth}
        \centering
        \includegraphics[width=\textwidth]{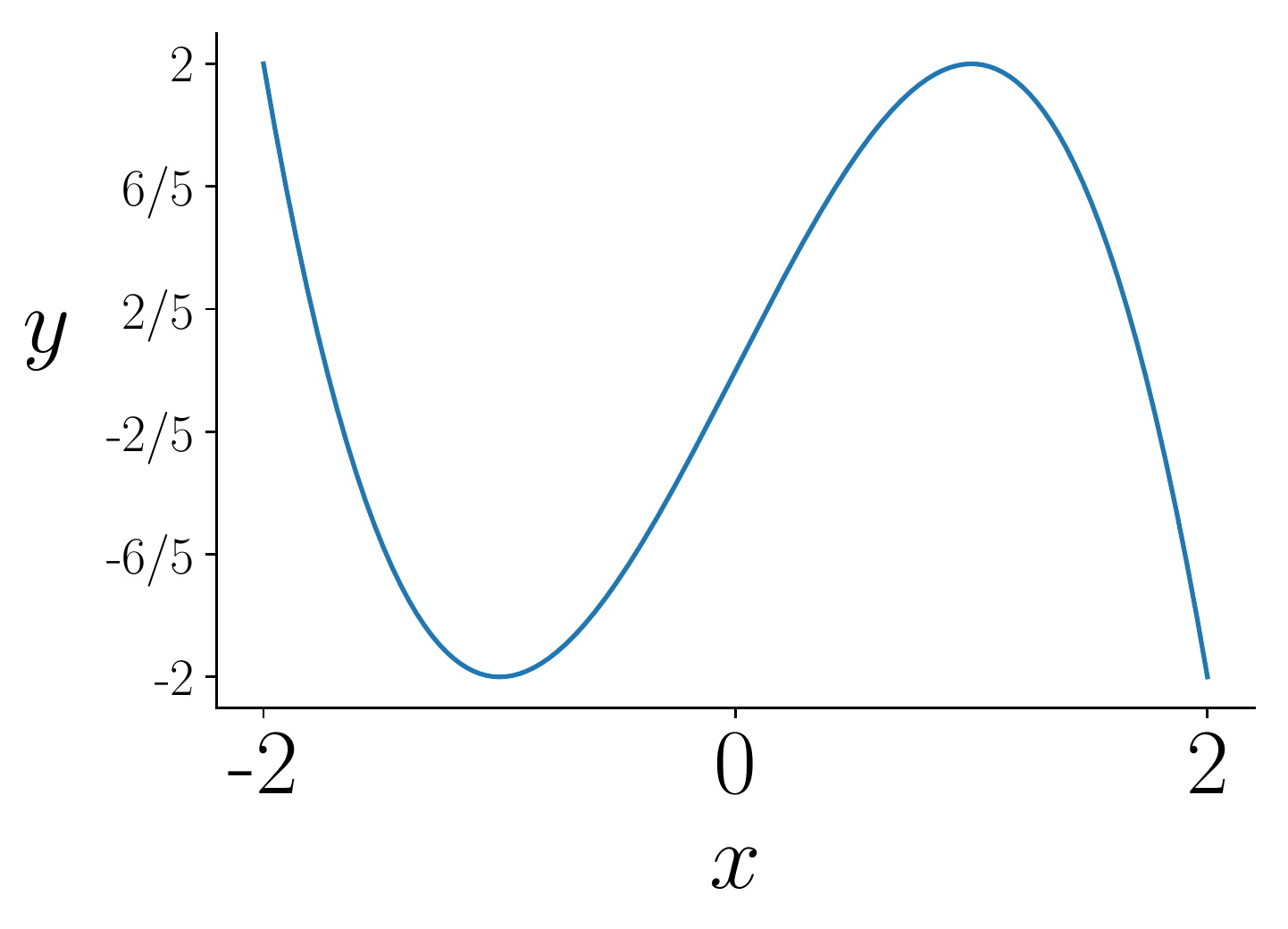}
        \subcaption{$f$}
        \label{fig:function}
    \end{subfigure} 
    \begin{subfigure}[t]{.32\textwidth}
        \centering
        \includegraphics[width=\textwidth]{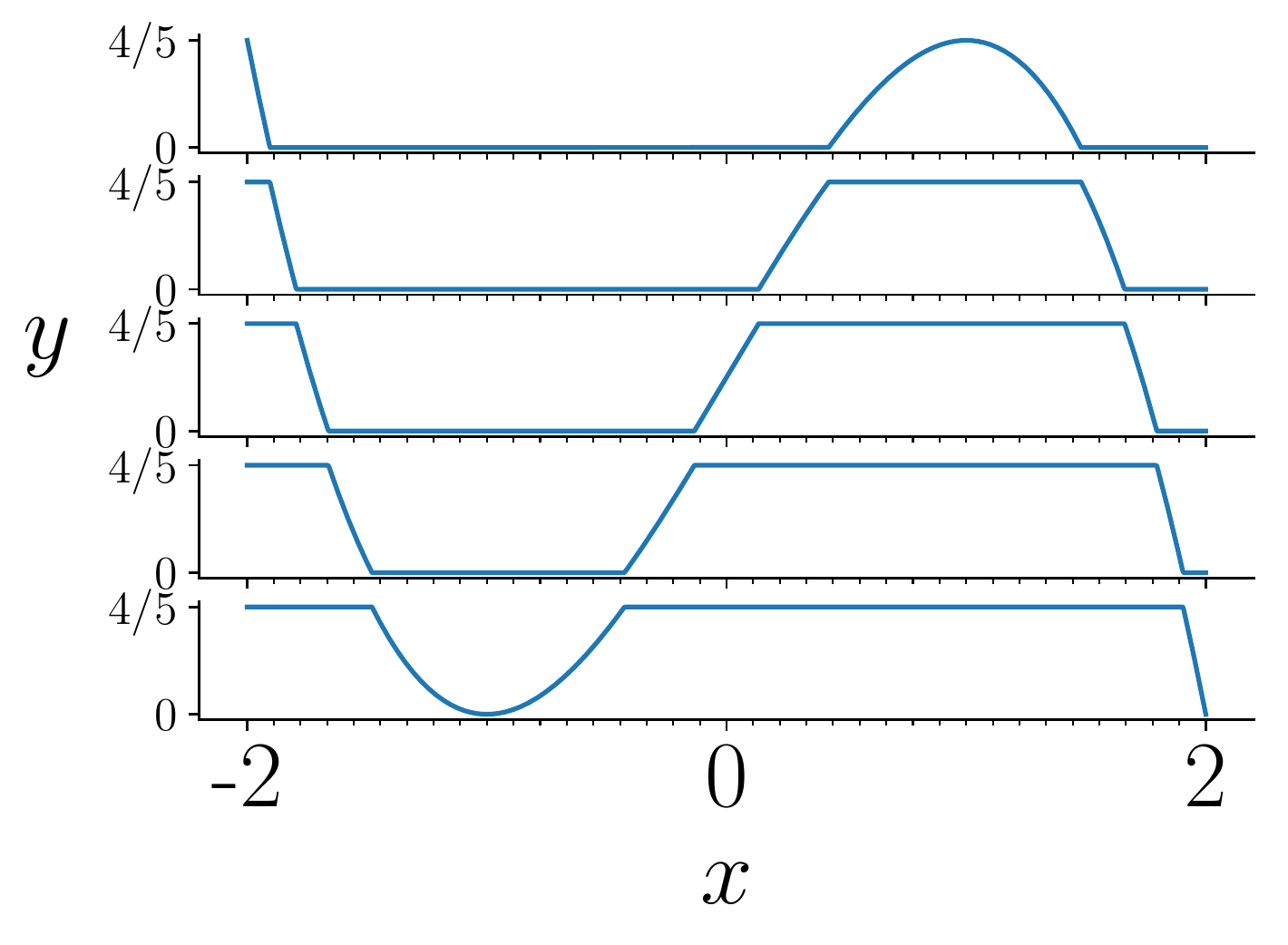}
        \subcaption{Slicing of $f$, $f_0,\ldots,f_4$}
        \label{fig:sliced_function}
    \end{subfigure}
    \begin{subfigure}[t]{.32\textwidth}
        \centering
        \includegraphics[width=\textwidth]{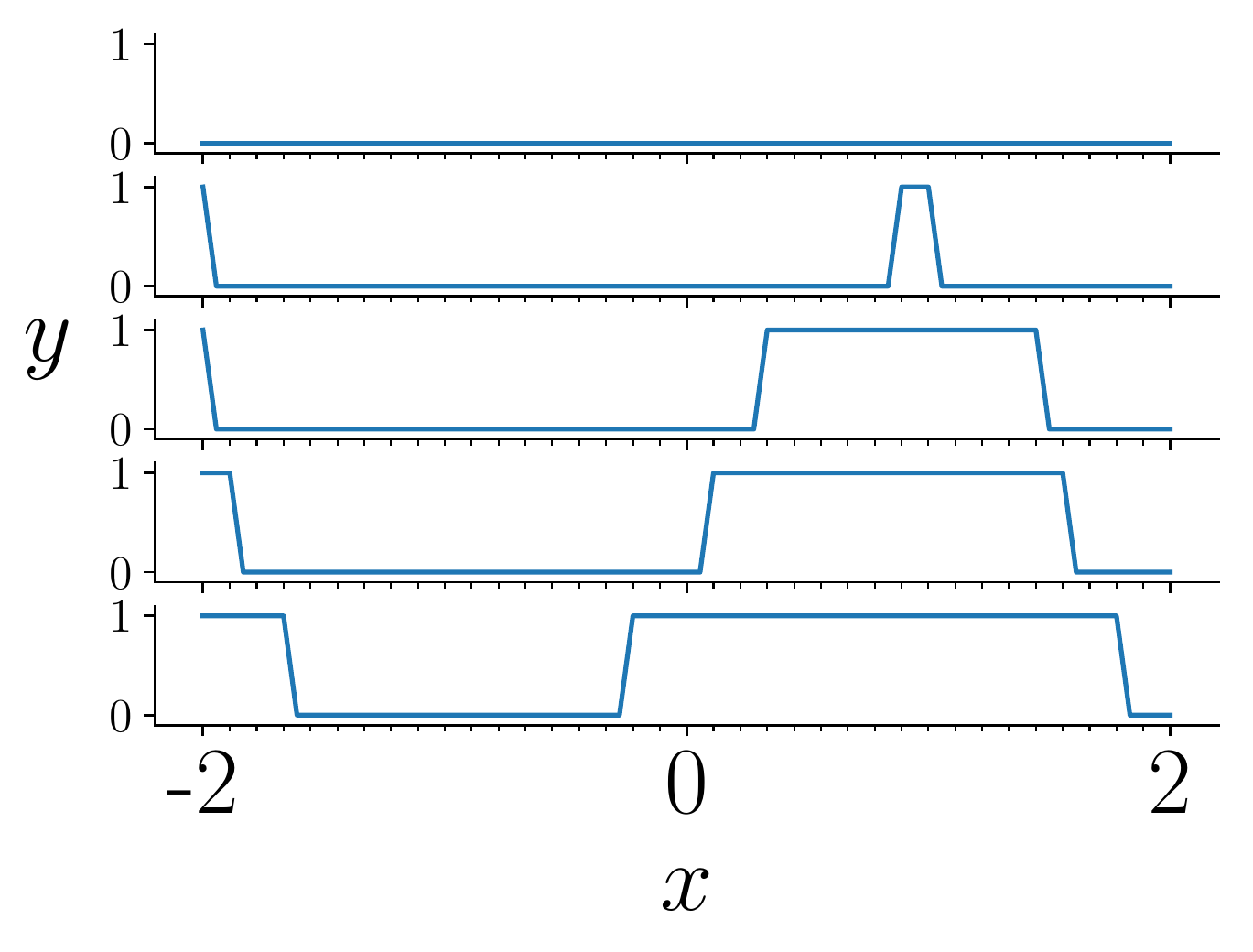}
        \subcaption{Networks $n_k$ approximating $f_k$.}
        \label{fig:local_bumps}
    \end{subfigure} 
    \caption{Approximating $f$ (\cref{fig:function}) using a ReLU network $n =
        \xi_0 + \sum_k n_k$. The ReLU networks $n_k$ (\cref{fig:local_bumps})
        approximate the $N$-slicing of $f$ (\cref{fig:sliced_function}), as a
        sum of local bumps (\cref{fig:local_bump}). }
        \label{fig:slicing}
\end{figure*}

\section{Proving Universal Interval-Provable Approximation} \label{sec:universal_approx_thm}

In this section, we provide an explanation of the proof of our main result,
\cref{thm:universal_approx_thm}, and illustrate the main points of the proof.

The first step in the construction is to deconstruct the function $f$ into
slices $\{f_k \colon \Gamma \to [0, \tfrac{\delta}{2}]\}_{0 \leq k < N}$ such
that that $f(x)=\xi_0+\sum_{k=0}^{N-1}f_k(x)$ for all $x$, where $\xi_{0}$ is
the minimum of $f(\Gamma)$. 
We approximate each slice $f_k$ by a ReLU network $\tfrac{\delta}{2} \cdot n_k$.
The network $n$ approximating $f$ up to $\delta$ will be
${n(x):=\xi_0+\tfrac{\delta}{2}\sum_{k}n_k(x)}$.  
The construction relies on 2 key insights, 
(i) the output of $\tfrac{\delta}{2} \cdot n_k^\sharp$ can be confined to the
interval $[0, \tfrac{\delta}{2}]$, thus the loss of analysis precision is at most
the height of the slice, and 
(ii) we can construct the networks $n_k$ using local bump functions, such that
only 4 slices can contribute to the loss of analysis precision, two for the
lower interval bound, two for the upper one.

The slicing $\{f_k\}_{0 \leq k < 5}$ of the function $f \colon [-2,2] \to
\mathbb{R}$ (\cref{fig:function}), mapping $x$ to $f(x) = -x^3+3x$ is depicted
in \cref{fig:sliced_function}.
The networks $n_k$ are depicted in \cref{fig:local_bumps}. 
In this example, evaluating the interval-transformer of $n$, namely $n^\sharp$
on the box $B = [-1, 1]$ results into $n^\sharp([-1,1]) = [-2, 6/5]$ lies is
within the $\delta=\tfrac{8}{5}$ bound of $f([-1, 1]) = [-2, 2]$.

\begin{definition}[$N$-slicing (\cref{fig:sliced_function})] \label{def:const_N_slicing}
    Let $\Gamma \subset \mathbb{R}^m$ be a closed $m$-dimensional box and let $f
    \colon \Gamma \to \mathbb{R}$ be continuous. The \emph{$N$-slicing} of $f$
    is a set of functions $\{f_k\}_{0 \leq k < N}$ defined by
    \begin{equation*}
        f_k \colon \Gamma \to \mathbb{R},  \quad x \mapsto
        \begin{cases}
            0 & \text{if } f(x) \leq \xi_k,
            \\
            f(x) - \xi_k & \text{if } \xi_k < f(x) < \xi_{k+1},
            \\
            \xi_{k+1} - \xi_k & \text{if } \xi_{k+1} \leq f(x),
        \end{cases}
        \qquad \forall k \in \{0, \dots, N-1\},
    \end{equation*}
    where $\xi_k := \xi_{0} + \frac{k}{N}(\xi_N - \xi_{0})$, $k \in \{1, \dots,
    N-1\}$, $\xi_{0} := \min f(\Gamma)$ and $\xi_N := \max f(\Gamma)$.
\end{definition}

\begin{figure}
    \centering
    \begin{minipage}{.48\textwidth}
        \centering
        \begin{tikzpicture}[scale=0.6, every node/.style={scale=0.8}]
            \pgfmathsetmacro{\l}{1.5}
            \pgfmathsetmacro{\h}{3}
            
            \foreach \i in {1,...,5}
                \foreach \j in {1,...,5} {
                    \coordinate (\j \i) at (\i * \l, \j * \l);
                }
        
            \coordinate (c) at (3 * \l, \h, 3 * \l);

            \foreach \i in {1,...,4}
                \foreach \j in {1,...,5} {
                    \draw[dotted] (\i * \l, \j * \l) -- (\i * \l + \l, \j * \l);
                }
            
            \foreach \i in {1,...,5}
                \foreach \j in {1,...,4} {
                    \draw[dotted] (\i * \l, \j * \l) -- (\i * \l, \j * \l + \l);
                }
            
            \draw[blue] (4.45 * \l, 4.45 * \l) -- (4.55 * \l, 4.55 * \l);
            \draw[blue] (4.45 * \l, 4.55 * \l) -- (4.55 * \l, 4.45 * \l);
            \node[blue] at (4.7 * \l, 4.5 * \l){$x$};
        
            \node at (4 * \l, 4 * \l)[circle,fill=blue,inner sep=1.5pt]{};
            \node at (4 * \l, 5 * \l)[circle,fill=blue,inner sep=1.5pt]{};
            \node at (5 * \l, 4 * \l)[circle,fill=blue,inner sep=1.5pt]{};
            \node at (5 * \l, 5 * \l)[circle,fill=blue,inner sep=1.5pt]{};
        
            \node[red] at (2.2 * \l, 1.7 * \l){$U$};
        
            \draw [fill=red,red] (0.9 * \l , 0.9 * \l) rectangle (1.1 * \l, 1.1 * \l);
            \draw [fill=red,red] (1.9 * \l , 0.9 * \l) rectangle (2.1 * \l, 1.1 * \l);
            \draw [fill=red,red] (2.9 * \l , 0.9 * \l) rectangle (3.1 * \l, 1.1 * \l);
            \draw [fill=red,red] (0.9 * \l , 1.9 * \l) rectangle (1.1 * \l, 2.1 * \l);
            \draw [fill=red,red] (2.9 * \l , 1.9 * \l) rectangle (3.1 * \l, 2.1 * \l);
            \draw [fill=red,red] (2.9 * \l , 2.9 * \l) rectangle (3.1 * \l, 3.1 * \l);
            \draw [fill=red,red] (1.9 * \l , 2.9 * \l) rectangle (2.1 * \l, 3.1 * \l);
            \draw [fill=red,red] (0.9 * \l , 2.9 * \l) rectangle (1.1 * \l, 3.1 * \l);
            
            \draw [red] plot [] coordinates {
                (1.3*\l, 1.3*\l)
                (2.7*\l, 1.3*\l)
                (2.7*\l, 2.7*\l)
                (1.3*\l, 2.7*\l)
                (1.3*\l, 1.3*\l)
                };
        
        \end{tikzpicture} 
        \captionof{figure}{Neighbors $\mathcal{N}(x)$ (blue dots) and
        $\mathcal{N}(U)$ (red squares).}
        \label{fig:neighbours}
    \end{minipage}
    \begin{minipage}{.48\textwidth}
        \vspace{-0.5cm}
        \centering
        \includegraphics[scale=0.4]{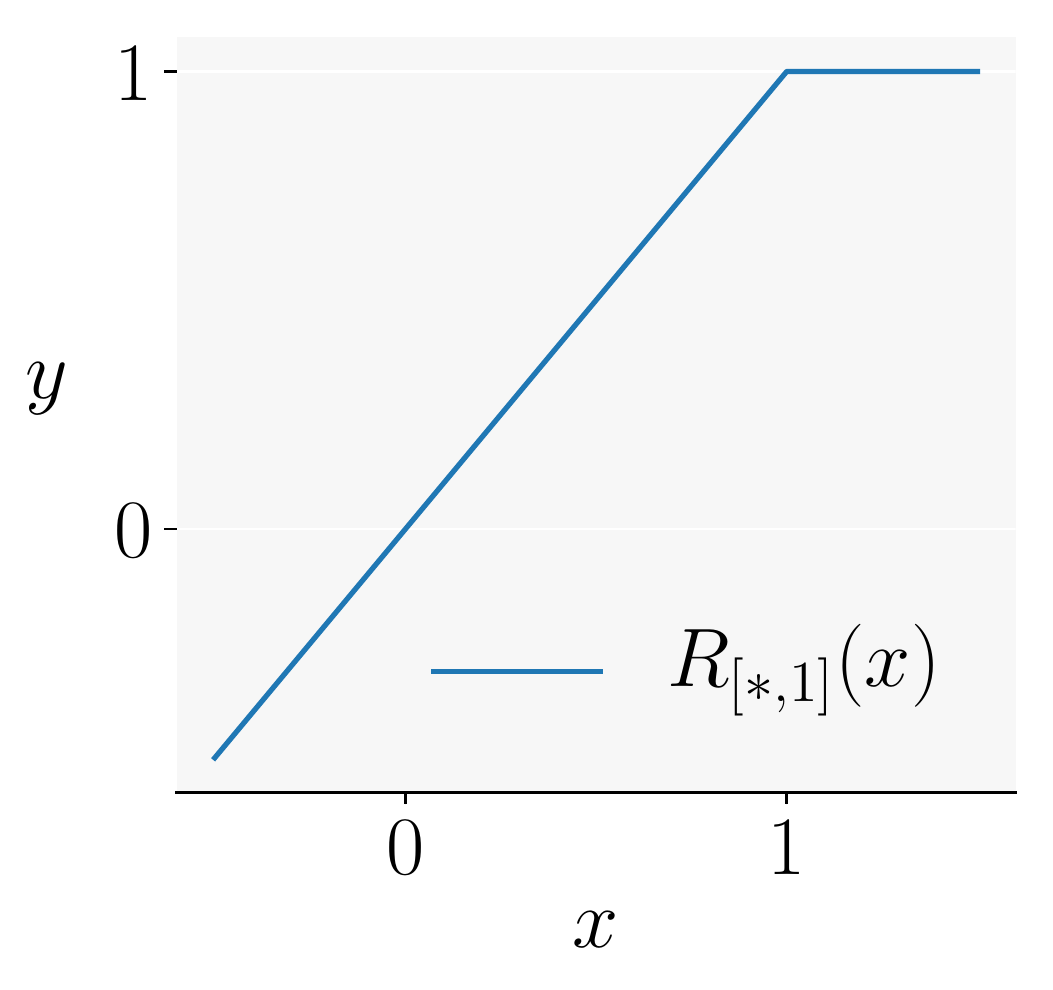}
        \vspace{-0.18cm}
        \captionof{figure}{$R_{[*,b]}(x)$}
        \label{fig:b_clipping}
    \end{minipage}
\end{figure}

To construct a ReLU network satisfying the desired approximation property
(\cref{eq:introduction}) if evaluated on boxes in $\mathcal{B}(\Gamma)$, we need
the ReLU network $\text{nmin}$ capturing the behavior of $\min$ as a building
block (similar to \cite{ReLU_DNN_FiniteElements}). It is given by

\begin{equation*}
    \text{nmin}(x, y) := \frac{1}{2}
    \begin{pmatrix}
        1 & -1 & -1 & -1
    \end{pmatrix}
    R \left(
        \begin{pmatrix}
            1 & 1 \\
            -1 & -1 \\
            1 & -1 \\
            -1 & 1 \\
        \end{pmatrix}
        \begin{pmatrix}
            x \\
            y
        \end{pmatrix}
    \right).
\end{equation*}

With the ReLU network $\text{nmin}$, we can construct recursively a ReLU network
$\text{nmin}_N$ mapping $N$ arguments to the smallest one (\cref{def:app:nmin}).
Even though the interval-transformation loses precision, we can establish bounds
on the precision loss of $\text{nmin}_N^\sharp$ sufficient for our use case
(\cref{app:universal_approx_thm}).

Now, we use the clipping function $R_{[*,1]} := 1 - R(1 - x)$ clipping every
value exceeding $1$ back to $1$ (\cref{fig:b_clipping}) to construct the local
bumps $\phi_c$ w.r.t. a grid $G$. $G$ specifies the set of all possible local
bumps we can use to construct the networks $n_k$. Increasing the finesse of $G$
will increases the approximation precision. 

\begin{definition}[local bump, \cref{fig:local_bump}] \label{def:local_bump}
    Let $M \in \mathbb{N}$, $G := \{ (\tfrac{i_1}{M}), \dots, \tfrac{i_m}{M}
    \mid i \in \mathbb{Z}^m \}$ be a grid, $\ell = 2^{\lceil \log_2 2m
    \rceil+1}$ and let $c = \{\tfrac{i_1^l}{M}, \tfrac{i_1^u}{M}\} \times \cdots
    \times \{\tfrac{i_m^l}{M}, \tfrac{i_m^u}{M}\} \subseteq G$ be a set of grid
    points describing the corner points of a hyperrectangle in $G$. We define a
    ReLU neural network $\phi_c \colon \mathbb{R}^m \to [0, 1] \subset
    \mathbb{R}$ w.r.t. $G$ by
    \begin{equation*}
        \phi_c(x)
        := R \left(\text{nmin}_{2m}
            \bigcup_{1 \leq k \leq m}
            \left\{
                \begin{matrix}
                    R_{[*,1]}(M \cdot \ell \cdot (x_k - \tfrac{i_k^l}{M}) + 1),  \\
                    R_{[*,1]}(M \cdot \ell \cdot (\tfrac{i_k^u}{M} - x_k) + 1)
                \end{matrix}
            \right\}
        \right).
    \end{equation*}
\end{definition}

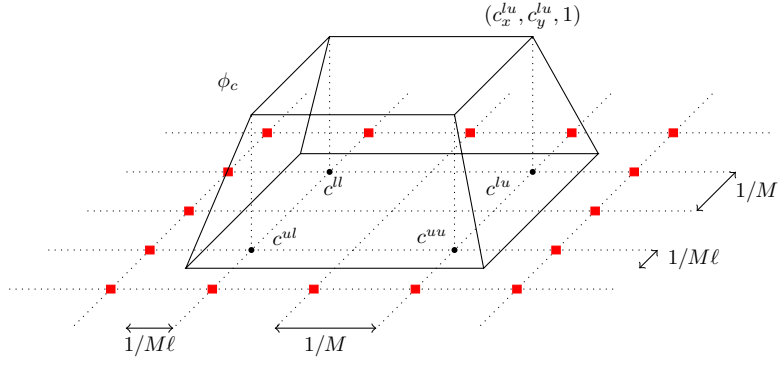
\begin{figure}
    \centering
    \begin{tikzpicture}[scale=0.9, every node/.style={scale=0.8}]
        \pgfmathsetmacro{\l}{1.5}
        \pgfmathsetmacro{\h}{2}
        \pgfmathsetmacro{\o}{0.7}
        
        \foreach \i in {1,...,8}
            \foreach \j in {1,...,8} {
                \coordinate (\j \i) at (\i * \l, 0, \j * \l);
            }

        \coordinate (c33) at (3 * \l, \h, 3 * \l);
        \coordinate (c35) at (5 * \l, \h, 3 * \l);
        \coordinate (c53) at (3 * \l, \h, 5 * \l);
        \coordinate (c55) at (5 * \l, \h, 5 * \l);

        \coordinate (b33) at (3 * \l - \o, 0, 3 * \l - \o);
        \coordinate (b35) at (5 * \l + \o, 0, 3 * \l - \o);
        \coordinate (b53) at (3 * \l - \o, 0, 5 * \l + \o);
        \coordinate (b55) at (5 * \l + \o, 0, 5 * \l + \o);

        \foreach \i in {1,...,6}
            \foreach \j in {2,...,6} {
                \draw[dotted] (\i * \l, 0 , \j * \l) -- (\i * \l + \l, 0 , \j * \l);
                \draw[dotted] (\j * \l, 0 , \i * \l) -- (\j * \l, 0 , \i * \l + \l);
            }

        \draw [fill=red,red] (5.92 * \l , 0, 5.9 * \l) rectangle (6.08 * \l, 0, 6.1 * \l);
        \draw [fill=red,red] (5.92 * \l , 0, 4.9 * \l) rectangle (6.08 * \l, 0, 5.1 * \l);
        \draw [fill=red,red] (5.92 * \l , 0, 3.9 * \l) rectangle (6.08 * \l, 0, 4.1 * \l);
        \draw [fill=red,red] (5.92 * \l , 0, 2.9 * \l) rectangle (6.08 * \l, 0, 3.1 * \l);
        \draw [fill=red,red] (5.92 * \l , 0, 1.9 * \l) rectangle (6.08 * \l, 0, 2.1 * \l);

        \draw [fill=red,red] (1.92 * \l , 0, 5.9 * \l) rectangle (2.08 * \l, 0, 6.1 * \l);
        \draw [fill=red,red] (1.92 * \l , 0, 4.9 * \l) rectangle (2.08 * \l, 0, 5.1 * \l);
        \draw [fill=red,red] (1.92 * \l , 0, 3.9 * \l) rectangle (2.08 * \l, 0, 4.1 * \l);
        \draw [fill=red,red] (1.92 * \l , 0, 2.9 * \l) rectangle (2.08 * \l, 0, 3.1 * \l);
        \draw [fill=red,red] (1.92 * \l , 0, 1.9 * \l) rectangle (2.08 * \l, 0, 2.1 * \l);

        \draw [fill=red,red] (2.92 * \l , 0, 1.9 * \l) rectangle (3.08 * \l, 0, 2.1 * \l);
        \draw [fill=red,red] (3.92 * \l , 0, 1.9 * \l) rectangle (4.08 * \l, 0, 2.1 * \l);
        \draw [fill=red,red] (4.92 * \l , 0, 1.9 * \l) rectangle (5.08 * \l, 0, 2.1 * \l);
        
        \draw [fill=red,red] (2.92 * \l , 0, 5.9 * \l) rectangle (3.08 * \l, 0, 6.1 * \l);
        \draw [fill=red,red] (3.92 * \l , 0, 5.9 * \l) rectangle (4.08 * \l, 0, 6.1 * \l);
        \draw [fill=red,red] (4.92 * \l , 0, 5.9 * \l) rectangle (5.08 * \l, 0, 6.1 * \l);

        \draw[dotted] (33) -- (c33) (35) -- (c35) (53) -- (c53) (55) -- (c55);
        \draw (b33) -- (c33)
              (b35) -- (c35)
              (b53) -- (c53)
              (b55) -- (c55)
              (c33) -- (c35) -- (c55) -- (c53) -- (c33)
              (b33) -- (b35) -- (b55) -- (b53) -- (b33)
              ;

        \draw[<->] (3 * \l - \o, 0, 7 * \l) -- node[below]{$1/M \ell$} (3 * \l, 0, 7 * \l);
        \draw[<->] (4 * \l, 0, 7 * \l) -- node[below]{$1/M $} (5 * \l, 0, 7 * \l);

        \draw[<->] (7 * \l, 0, 5 * \l + \o) -- node[right]{$\;\; 1/M \ell$} (7 * \l, 0, 5 * \l);
        \draw[<->] (7 * \l, 0, 4 * \l) -- node[right]{$\;\; 1/M$} (7 * \l, 0, 3 * \l);

        \draw (3 * \l, 0, 3 * \l + 0.35 * \l) node[right]{$c^{ll}$};
        \node at (3 * \l, 0, 3 * \l)[circle,fill=black,inner sep=1pt]{};
        \draw (5 * \l, 0, 3 * \l + 0.35 * \l) node[left]{$c^{lu}$};
        \node at (5 * \l, 0, 3 * \l)[circle,fill=black,inner sep=1pt]{};
        \draw (3 * \l, 0, 5 * \l - 0.35 * \l) node[right]{$c^{ul}$};
        \node at (3 * \l, 0, 5 * \l)[circle,fill=black,inner sep=1pt]{};
        \draw (5 * \l - 0.2, 0, 5 * \l - 0.35 * \l) node[left]{$c^{uu}$};
        \node at (5 * \l, 0, 5 * \l)[circle,fill=black,inner sep=1pt]{};

        \draw (c35) node[above]{$(c^{lu}_x, c^{lu}_y, 1)$};

        \draw (2 * \l, \h * 2 / 3, 3 * \l) node[]{$\phi_c$};

    \end{tikzpicture} 
    \caption{Local bump $\phi_c$, where $c$ contains the points $c^{ll}, c^{lu},
    c^{ul}, c^{uu}$. The points in $\mathcal{N}(\text{conv}(c))$ are depicted
    by the red squares.}
    \label{fig:local_bump} 
\end{figure}

We will describe later how $M$ and $c$ get picked. A graphical illustration of a
local bump for in two dimensions and $c = \{\tfrac{i_1^l}{M}, \tfrac{i_1^u}{M}\}
\times \{\tfrac{i_2^l}{M}, \tfrac{i_2^u}{M}\} = \{c^{ll}, c^{lu}, c^{ul},
c^{uu}\}$ is shown in \cref{fig:local_bump}.
The local bump $\phi_c(x)$ evaluates to 1 for all $x$ that lie within the convex
hull of $c$, namely $\text{conv}(c)$, after which $\phi_c(x)$ quickly decreases
linearly to 0. $\phi_c$ has $1+2(2d-1)+2d$ ReLUs and $1+\lceil \log_2(2d+1)
\rceil +1$ layers.

By construction $\phi_c(x)$ decreases to 0 before reaching the next neighboring
grid points $\mathcal{N}(\text{conv}(c))$, where $\mathcal{N}(x) := \{g \in G
\mid ||x - g||_\infty \leq \tfrac{1}{M}\} \setminus \{x\}$ denotes the
neighboring grid points of $x$ and similarly for $\mathcal{N}(U) :=
\{\mathcal{N}(x) \mid x \in U\} \setminus U$ (\cref{fig:neighbours}). The set
$\mathcal{N}(\text{conv}(c))$ forms a hyperrectangle in $G$ and is shown in
\cref{fig:local_bump} using red squares. Clearly $\text{conv}(c) \subseteq
\text{conv}(\mathcal{N}(c))$. 

Next, we give bounds on the loss of precision for the interval-transformation
$\phi_c^\sharp$. We can show that interval analysis can (i) never produce
intervals exceeding $[0, 1]$ and (ii) is precise if $B$ does no intersect
$\text{conv}(\mathcal{N}(c)) \setminus \text{conv}(c)$.

\begin{lemma} \label{lem:local_bump_abstract}
    For all $B \in \mathcal{B}(\mathbb{R}^m)$, it holds that $\phi_c^\sharp(B)
    \subseteq [0,1] \in \mathcal{B}$ and 
    \begin{equation*}
        \phi_c^\sharp(B)
        =
        \begin{cases}
            [1, 1] &\text{if } B \subseteq \text{conv}(c)
            \\
            [0, 0] &\text{if } B \subseteq \Gamma \setminus \text{conv}(\mathcal{N}(c)).
        \end{cases}
    \end{equation*}
\end{lemma}

The formal proof is given in \cref{app:universal_approx_thm}. 
The next lemma shows, how a ReLU network $n_k$ can approximate the slice $f_k$
while simultaneously confining the loss of analysis precision. 

\begin{lemma} \label{lem:Rn_slice}
    Let $\Gamma \subset \mathbb{R}^m$ be a closed box and let $f \colon \Gamma
    \to \mathbb{R}$ be continuous. For all $\delta > 0$ there exists a set of
    ReLU networks $\{n_k\}_{0 \leq k < N}$ of size $N \in \mathbb{N}$
    approximating the $N$-slicing of $f$, $\{f_k\}_{0 \leq k < N}$ ($\xi_k$ as
    in \cref{def:const_N_slicing}) such that for all boxes $B \in
    \mathcal{B}(\Gamma)$
    \begin{equation}
        n_k^\sharp(B) =
        \begin{cases}
            [0, 0] 
            &\text{if } f(B) \leq \xi_k - \tfrac{\delta}{2}
            \\
            [1, 1] 
            &\text{if } f(B) \geq \xi_{k+1} + \tfrac{\delta}{2}. 
        \end{cases}
        \label{eq:lem:Rn_slices}
    \end{equation}
    and $n_k^\sharp(B) \subseteq [0, 1]$. 
\end{lemma}

It is important to note that in \cref{eq:lem:Rn_slices} we mean $f$ and not
$f^\sharp$. The proof for \cref{lem:Rn_slice} is given in
\cref{app:universal_approx_thm}. In the following, we discuss a proof sketch. 

Because $\Gamma$ is compact and $f$ is continuous, $f$ is uniformly continuous
by the Heine-Cantor Theorem. So we can pick a $M \in \mathbb{N}$ such that for
all $x, y \in \Gamma$ satisfying $||y - x||_\infty \leq \tfrac{1}{M}$ holds
$|f(y) - f(x)| \leq \tfrac{\delta}{2}$. We then choose the grid $G =
(\frac{\mathbb{Z}}{M})^m \subseteq \mathbb{R}^m$. 

Next, we construct for every slice $k$ a set $\Delta_k$ of hyperrectangles on
the grid $G$: if a box $B \in \mathcal{B}(\Gamma)$ fulfills $f(B) \geq
\xi_{k+1}+\tfrac{\delta}{2}$, then we add a minimal enclosing hyperrectangle $c
\subset G$ such that $B \subseteq \text{conv}(c)$ to $\Delta_k$, where
$\text{conv}(c)$ denotes the convex hull of $c$. This implies, using uniform
continuity of $f$ and that the grid $G$ is fine enough, that $f(\text{conv}(c))
\geq \xi_{k+1}$. Since there is only a finite number of possible hyperrectangles
in $G$, the set $\Delta_k$ is clearly finite. The network fulfilling
\cref{eq:lem:Rn_slices} is
\begin{equation*}
    n_k(x) := R_{[*,1]}\left(
        \sum_{c \in \Delta_k} \phi_c(x)\right),
\end{equation*}
where $\phi_c$ is as in \cref{def:local_bump}. The $n_k$ are depicted in
\cref{fig:local_bumps}.

Now, we see that \cref{eq:lem:Rn_slices} holds by construction: For all boxes $B \in
\mathcal{B}(\Gamma)$ such that $f \geq \xi_{k+1} + \tfrac{\delta}{2}$ on $B$
exists $c' \in \Delta_k$ such that $B \subseteq \text{conv}(c')$ which implies,
using \cref{lem:local_bump_abstract}, that $\phi_{c'}^\sharp(B) = [1,1]$, hence
\begin{alignat*}{3}
    n_k^\sharp(B)
    &= R_{[*,1]}^\sharp(
        \phi_{c'}^\sharp (B)
        + \sum_{c \in \Delta_k \setminus c'} \phi_c^\sharp(B)) 
    \qquad \qquad
    && \forall c \neq c' : \phi_c^\sharp(B) \subseteq [0,1] \text{(\cref{lem:local_bump_abstract})}
    \\
    &= R_{[*,1]}^\sharp( [1,1] + [p_1, p_2]) 
    \qquad
    && [p_1, p_1] \in \mathcal{B}(\mathbb{R}_{\geq 0})
    \\
    &= R_{[*,1]}^\sharp( [1+p_1,1+p_2])
    &&
    \\
    &= [1, 1].
\end{alignat*}

Similarly, if $f(B) \leq \xi_k - \tfrac{\delta}{2}$ holds, then it holds for all
$c \in \Delta_k$ that $B$ does not intersect $\mathcal{N}(\text{conv}(c))$.
Indeed, if a $c \in \Delta_k$ would violate this, then by construction,
$f(\text{conv}(c)) \geq \xi_{k+1}$, contradicting $f(B) \leq \xi_k -
\tfrac{\delta}{2}$. Thus $\phi_c^\sharp(B) = [0, 0]$, and hence $n^\sharp(B) =
[0, 0]$. 

\begin{theorem} \label{thm:universal_approx_thm_prelim}
    Let $\Gamma \subset \mathbb{R}^m$ be a closed box and let $f \colon \Gamma
    \to \mathbb{R}$ be continuous. Then for all $\delta > 0$,
    exists a ReLU network $n$ such that for all $B \in \mathcal{B} (\Gamma)$ 
    \begin{equation*}
        [l+\delta, u-\delta] \subseteq n^\sharp(B) \subseteq [l-\delta, u+\delta],
    \end{equation*}
    where $l := \min f(B)$ and $u := \max f(B)$. 
\end{theorem}

\textit{Proof. }
    Pick $N$ such that the height of each slice is exactly $\tfrac{\delta}{2}$,
    if this is impossible choose a slightly smaller $\delta$. 
    Let $\{n_k\}_{0 \leq k < N}$ be a series of networks as in
    \cref{lem:Rn_slice}. Recall that $\xi_{0} = \min f(\Gamma)$. We
    define the ReLU network
    \begin{equation}
        n(x) := \xi_{0} + \tfrac{\delta}{2} \sum_{k=0}^{N-1} n_k(x).
        \label{eq:final_network}
    \end{equation}
    Let $B \in \mathcal{B}(\Gamma)$. Thus we have for all $k$
    \begin{alignat}{2}
        f(B) \geq \xi_{k+2}
        &\Leftrightarrow
        f(B) \geq \xi_{k+1} + \tfrac{\delta}{2}
        \quad
        &&\overset{\cref{lem:Rn_slice}}{\Rightarrow}
        \quad
        n_k^\sharp(B) = [1, 1]
        \label{eq:thm:proof:1}
        \\
        f(B) \leq \xi_{k-1}
        &\Leftrightarrow
        f(B) \leq \xi_k - \tfrac{\delta}{2}
        \quad
        &&\overset{\cref{lem:Rn_slice}}{\Rightarrow}
        \quad
        n_k^\sharp(B) = [0, 0].
        \label{eq:thm:proof:2} 
    \end{alignat}
    Let $p, q \in \{0,\dots,N-1\}$ such that 
    \begin{align}
        \xi_p \leq l = \min f(B) \leq \xi_{p+1} 
        \label{eq:thm:cond1} 
        \\
        \xi_q \leq u = \max f(B) \leq \xi_{q+1},
        \label{eq:thm:cond2}
    \end{align}
    \begin{wrapfigure}{r}{0.33\textwidth} 
        \centering
        \vspace{-0.4cm}
        \includegraphics[width=0.235\textwidth]{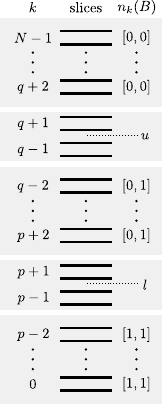}
        \caption{Illustration of the proof for \cref{thm:universal_approx_thm_prelim}.}
        \label{fig:proof_expl}
        \vspace{-3cm}
    \end{wrapfigure}
    as depicted in \cref{fig:proof_expl}. Thus by \cref{eq:thm:proof:1} for all
    ${k \in \{0, \dots, p-2\}}$ it holds that $n_k^\sharp(B) = [1, 1]$ and
    similarly, by \cref{eq:thm:proof:2} for all $k \in \{q+2, \dots, N-1 \}$ it
    holds that $n_k^\sharp(B) = [0, 0]$. Plugging this into
    \cref{eq:final_network} after splitting the sum into three parts leaves us
    with
    \begin{align*}
        n^\sharp(B) 
        &= \xi_{0} + \tfrac{\delta}{2} \sum_{k=0}^{p-2} n_k^\sharp(B) 
                    + \tfrac{\delta}{2} \sum_{k=p-1}^{q+1} n_k^\sharp(B) 
                    + \tfrac{\delta}{2} \sum_{k=p+1}^{N-1} n_k^\sharp(B)
        \\
        &= \xi_{0} 
            + (p-1) [\tfrac{\delta}{2}, \tfrac{\delta}{2}] 
            + \tfrac{\delta}{2} \sum_{k=p-1}^{q+1} n_k^\sharp(B) + [0, 0].
    \end{align*}
    Applying the standard rules for interval analysis, leads to 
    \begin{equation*}
        n^\sharp(B) = [\xi_{p-1}, \xi_{p-1}] + \tfrac{\delta}{2} \sum_{k=p-1}^{q+1} n_k^\sharp(B),
    \end{equation*}
    where we used in the last step, that $\xi_{0} + k \tfrac{\delta}{2} =
    \xi_k$. For all terms in the sum except the terms
    corresponding to the 3 highest and lowest $k$ we get 
    \begin{equation}
        n_k^\sharp(B) = [0, 1] \qquad \forall k \in \{p+2, \dots, q-2\}.
    \end{equation}
    Indeed, from \cref{eq:thm:cond1} we know that there is $x \in B$ such that
    $f(x) \leq \xi_{p+1} = \xi_{p+2} - \tfrac{\delta}{2}$, thus by
    \cref{lem:Rn_slice} $n_k^\sharp([x, x]) = [0, 0]$ for all $p+2 \leq k \leq
    q-2$. Similarly, from \cref{eq:thm:cond2} we know, that there is $x' \in B$
    such that $f(x) \geq \xi_q = \xi_{q-1} + \tfrac{\delta}{2}$, thus by
    \cref{lem:Rn_slice} $n_k^\sharp([x',x']) = [1,1]$ for all $p+2 \leq k \leq
    q-2$. So $n_k^\sharp(B)$ is at least $[0,1]$, and by \cref{lem:Rn_slice}
    also at most $[0,1]$. This leads to
    \begin{alignat*}{3}
        n^\sharp(B)
        &= [\xi_{p-1}, \xi_{p-1}] +
        \tfrac{\delta}{2} \sum_{k=p-1}^{p+1} n_k^\sharp(B) +
        \tfrac{\delta}{2} ((q-2)-(p+2) + 1) [0,1]
        && + \tfrac{\delta}{2} \sum_{k=q-1}^{q+1} n_k^\sharp(B)
        \\
        &= [\xi_{p-1}, \xi_{p-1}] +
            \tfrac{\delta}{2} \sum_{k=p-1}^{p+1} n_k^\sharp(B) +
            [0, \xi_{q-1} - \xi_{p+2}] 
        && + \tfrac{\delta}{2} \sum_{k=q-1}^{q+1} n_k^\sharp(B).
    \end{alignat*}
    We know further, that if $p + 3 \leq q$, than there is an $x \in B$ such
    that $f(x) \geq \xi_{p+3} = \xi_{p+2} + \tfrac{\delta}{2}$, hence similar as
    before $n_{p+1}^\sharp([x, x]) = [1,1]$ and similarly
    $n_{p}^\sharp([x,x])=[1,1]$ and $n^\sharp([x,x])=[1,1]$. So we know, that
    $\tfrac{\delta}{2} \sum_{k=p-1}^{p+1} n_k^\sharp(B)$ includes at least
    $[3\tfrac{\delta}{2},3\tfrac{\delta}{2}]$ and at the most $[0,
    3\tfrac{\delta}{2}]$. Similarly, there exists an $x' \in B$ such that
    $n_{q-1}^\sharp([x', x']) = [0,0]$, $n_{q}^\sharp([x', x']) = [0,0]$ and
    $n_{q+1}^\sharp([x', x']) = [0,0]$. This leaves us with
    \begin{align*}
        [3\tfrac{\delta}{2},3\tfrac{\delta}{2}] 
        \subseteq
        \tfrac{\delta}{2} \sum_{k=p-1}^{p+1} n_k^\sharp(B) 
        \subseteq 
        [0, 3\tfrac{\delta}{2}]
        \\
        [0,0] 
        \subseteq 
        \tfrac{\delta}{2} \sum_{k=q-1}^{q+1} n_k^\sharp(B) 
        \subseteq [0, 3\tfrac{\delta}{2}],
    \end{align*}
    If $p + 3 > q$ the lower bound we want to prove becomes vacuous and
    only the upper one needs to be proven. Thus we have
    \begin{equation*}
        [l + \delta, u - \delta] \subseteq [\xi_{p+2}, \xi_{p-1}] \subseteq 
        n^\sharp(B) \subseteq [\xi_{p-1, \xi_{q+2}}] \subseteq [l - \delta, u + \delta],
    \end{equation*}
    where $l := \min f(B)$ and $u := \max f(B)$.
\qed

\begin{theorem}[Universal Interval-Provable Approximation] \label{thm:universal_approx_thm}
    Let $\Gamma \subset \mathbb{R}^m$ be compact and $f \colon \Gamma \to
    \mathbb{R}^d$ be continuous. For all $\delta \in \mathbb{R}^m_{\geq 0}$
    exists a ReLU network $n$ such that for all $B \in \mathcal{B} (\Gamma)$ 
    \begin{equation*}
        [l+\delta, u-\delta] \subseteq n^\sharp(B) \subseteq [l-\delta, u+\delta],
    \end{equation*}
    where $l, u \in \mathbb{R}^m$ such that $l_k := \min f(B)_k$ and $u_k :=
    \max f(B)_k$ for all $k$. 
\end{theorem}

\begin{proof}
    This is a direct consequence of using \cref{thm:universal_approx_thm_prelim}
    and the Tietze extension theorem to produce a neural network for each
    dimension $d$ of the codomain of $f$.
\end{proof}

Note that \cref{thm:intro:universal_approx_thm} is a special case of
\cref{thm:universal_approx_thm} with $d=1$ to simplify presentation.

\section{Conclusion}
We proved that for all real valued continuous functions $f$ on compact sets,
there exists a ReLU network $n$ approximating $f$ arbitrarily well with the
interval abstraction. This means that for arbitrary input sets, analysis using
the interval relaxation yields an over-approximation arbitrarily close to the
smallest interval containing all possible outputs. Our theorem affirmatively
answers the open question, whether the Universal Approximation Theorem
generalizes to Interval analysis.

Our results address the question of whether the interval abstraction is
expressive enough to analyse networks approximating interesting functions $f$.
This is of practical importance because interval analysis is the most scalable
non-trivial analysis.

\message{^^JLASTBODYPAGE \thepage^^J}

\clearpage
\bibliography{references}
\bibliographystyle{plainnat}

\message{^^JLASTREFERENCESPAGE \thepage^^J}

\clearpage
\appendix

\section{Proofs for the Universal Interval-Certified Approximation} \label{app:universal_approx_thm}

\begin{lemma}[Monotonicity] The operations $+, -$ are monotone, that is for
    all $[a_1, b_1], [a_2, b_2], [c_1, d_1], [c_2, d_2]  \in \mathcal{B}(R)$ such
    that $[a_1, b_1] \subseteq [a_2, b_2]$ and $[c_1, d_2] \subseteq [c_2, d_2]$ holds
    \begin{align*}
        [a_1, b_1] +^\sharp [c_1, d_1] &\subseteq [a_2, d_2] +^\sharp [c_2, d_2] \\
        [a_1, b_1] -^\sharp [c_1, d_1] &\subseteq [a_2, d_2] -^\sharp [c_2, d_2] \\
        [a_1, b_1] \cdot^\sharp [c_1, d_1] &\subseteq [a_2, d_2] \cdot^\sharp [c_2, d_2]. \\
    \end{align*}
    Further the operation $*$ and $R$ are monotone, that is for all $[a, b], [c, d] \in
    \mathcal{B}(R)$ and for all $\lambda \in \mathbb{R}_{\geq 0}$  such that $[a,b]\subseteq[c,d]$ holds
    \begin{align*}
        \lambda \cdot^\sharp [a, b] &\subseteq \lambda \cdot^\sharp [c,d] \\
        R^\sharp([a, b]) &\subseteq R^\sharp([c, d]). 
    \end{align*} 
\end{lemma}

\begin{proof}
    \begin{align*}
        [a_1, b_1] +^\sharp [c_1, d_1] = [a_1 + c_1, b_1 + d_1] 
        &\subseteq [a_2 + c_2, b_2 + d_2] = [a_2, d_2] +^\sharp [c_2, d_2]
        \\
        [a_1, b_1] -^\sharp [c_1, d_1] = [a_1 - d_1, b_1 - c_1]
        &\subseteq [a_2 - d_2, b_2 - c_2] = [a_2, d_2] -^\sharp [c_2, d_2]
    \end{align*}
    
    \begin{align*}
        \lambda \cdot^\sharp [a, b] = [\lambda a, \lambda b] 
        &\subseteq [\lambda c, \lambda d] = [\lambda c, \lambda d]
        \\
        R^\sharp([a, b]) = [R(a), R(b)] &\subseteq [R(c), R(d)] = R^\sharp([c, d]).
    \end{align*}
\end{proof}

\begin{definition}[$N$-slicing] \label{def:app:const_N_slicing}
    Let $\Gamma \subset \mathbb{R}^m$ be a compact $m$-dimensional box and let
    $f \colon \Gamma \to \mathbb{R}$ be continuous. The \emph{$N$-slicing} of
    $f$ is a set of functions $\{f_k\}_{0 \leq k \leq N-1}$ defined by
    \begin{equation*}
        f_k \colon \Gamma \to \mathbb{R},  \quad x \mapsto 
        \begin{cases}
            0 & \text{if } f(x) \leq \xi_k,
            \\
            f(x) - \xi_k & \text{if } \xi_k < f(x) < \xi_{k+1},
            \\
            \xi_{k+1} - \xi_k & \text{ otherwise},
        \end{cases}
        \;\; \forall k \in \{0, \dots, N-1\},
    \end{equation*}
    where $\xi_k := \frac{k}{N}(\xi_{\max} - \xi_{\min})$, $k \in \{0, \dots,
    N\}$, $\xi_{\min} := \min f(\Gamma)$ and $\xi_{\max} := \max f(\Gamma)$. 
\end{definition}

\begin{lemma}[$N$-slicing] \label{lem:app:identity_func_N_slicing}
    Let $\{f_k\}_{0 \leq k \leq N-1}$ be the $N$-slicing of $f$. Then for all $x \in \Gamma$ we have $f(x) := \xi_0 + \sum_{k=0}^{N-1} f_k(x)$.  
\end{lemma}

\begin{proof}
    Pick $x \in \Gamma$ and let $l \in \{0,\dots,N-1\}$ such that $\xi_l \leq f(x) \leq \xi_{l+1}$. Then
    \begin{align*}
        \xi_0 + \sum_{k=0}^{N-1} f_k(x) &= \xi_0 + \sum_{k=0}^{l-1} f_k(x) + f_l(x) + \sum_{k=l+1}^{N-1} f_k(x)
        = \xi_0 + \sum_{k=0}^{l-1} (\xi_{k+1} - \xi_k) + f_l(x) 
        \\
        &= \xi_l + f_l(x) = f(x). 
    \end{align*}
\end{proof}

\begin{definition}[clipping] 
    Let $a, b \in \mathbb{R}$, $a < b$. We define the \emph{clipping} function  
    $R_{[*, b]} \colon \mathbb{R} \to \mathbb{R}$ by 
    \begin{align*}
        R_{[*, b]}(x) &:= b - R(b - x). 
    \end{align*}
\end{definition}

\begin{lemma}[clipping] \label{lem:app:alpha_beta_clipping}
    The function $R_{[*, b]}$ sends all $x \leq b$ to $x$, and all $x > b$ to
    $b$. Further, $R_{[*, b]}^\sharp([a', b']) = [R_{[*,b]}(a'),
    R_{[*,b]}(b')]$. 
\end{lemma}

\begin{proof}
    We show the proof for $R_{[a,b]}$, the proof for $R_{[*,b]}$ is similar. 
    \begin{align*}
        x < b
        & \Rightarrow
        R_{[*,b]}(x) = b - R(b - x) = b - b + x = x
        \\
        x \geq b
        & \Rightarrow
        R_{[*,b]}(x) = b - R(b - x) = b - 0 = b
    \end{align*}
    
    Next,
    \begin{align*}
        R_{[*, b]}^\sharp([a', b']) 
        &= b -^\sharp R^\sharp(b -^\sharp [a', b']) 
        \\
        &= b -^\sharp R^\sharp(b +^\sharp [-b', -a'])
        \\
        &= b -^\sharp R^\sharp([b - b', b - a']) 
        \\
        &= b -^\sharp [R(b - b'), R(b - a')] 
        \\
        &= b +^\sharp [- R(b - a'), - R(b - b')] 
        \\
        &= [b - R(b - a'), b - R(b - b')] 
        \\
        &= [R_{[*, b]}(a'), R_{[*, b]}(b')]. 
    \end{align*}
\end{proof}

\begin{definition}[nmin] We define the ReLU network $\text{nmin} \colon
    \mathbb{R}^2 \to \mathbb{R}$ by
    \begin{equation*}
        \text{nmin}(x, y) := \frac{1}{2} 
        \begin{pmatrix}
            1 & -1 & -1 & -1
        \end{pmatrix}
        R \left(
            \begin{pmatrix}
                1 & 1 \\
                -1 & -1 \\
                1 & -1 \\
                -1 & 1 \\
            \end{pmatrix}
            \begin{pmatrix}
                x \\
                y
            \end{pmatrix}
        \right). 
    \end{equation*}
\end{definition}

\begin{lemma}[nmin]\label{lem:app:nmin}
    Let $x, y \in \mathbb{R}$, then $\text{nmin}(x,y) = \min(x,y)$. 
\end{lemma}

\begin{proof}
    Because $\text{nmin}$ is symmetric in its arguments, we assume w.o.l.g. $x
    \geq y$. 
    \begin{align*}
        \text{nmin}(x, y) 
        &= \frac{1}{2} 
        \begin{pmatrix}
            1 & -1 & -1 & -1
        \end{pmatrix}
        R \left(
            \begin{pmatrix}
                1 & 1 \\
                -1 & -1 \\
                1 & -1 \\
                -1 & 1 \\
            \end{pmatrix}
            \begin{pmatrix}
                x \\
                y
            \end{pmatrix}
        \right)
        \\
        &= \frac{1}{2} 
        \begin{pmatrix}
            1 & -1 & -1 & -1
        \end{pmatrix}
        R \begin{pmatrix}
                x + y \\
                - x - y \\
                x - y \\
                - x + y \\
            \end{pmatrix}
    \end{align*}
    If $x + y \geq 0$, then
    \begin{align*}
        \text{nmin}(x, y) = \frac{1}{2} (x + y - x + y) = y. 
    \end{align*}
    If $x + y < 0$, then 
    \begin{align*}
        \text{nmin}(x, y) = \frac{1}{2} (x + y - x + y) = y. 
    \end{align*}
\end{proof}

\begin{definition}[$\text{nmin}_N$] \label{def:app:nmin}
    For all $N \in \mathbb{N}_{\geq 1}$, we
    define a ReLU network $\text{nmin}_N$ defined by 
    \begin{align*}
        \text{nmin}_1(x) &:= x
        \\
        \text{nmin}_N(x_1, \dots, x_N) &:= \text{nmin}(\text{nmin}_{\lceil N/2 \rceil}(x_1, \dots, x_{\lceil N/2 \rceil}), \text{nmin}_{\lceil N/2 \rceil + 1}(x_{\lceil N / 2 \rceil + 1}, \dots, x_N)). 
    \end{align*}
\end{definition}

\begin{lemma} \label{lem:app:min_abstract}
    Let $[a, b], [c, d] \in \mathcal{B}(\mathbb{R})$. Then
    $\text{nmin}^\sharp([a, b], [c, d]) = \text{nmin}^\sharp([c, d], [a, b])$
    and 
    \begin{equation*}
        \text{nmin}^\sharp([a, b], [c, d]) = 
        \begin{cases}
         [c + \tfrac{a - b}{2}, d + \tfrac{b - a}{2}] & \text{if } d \leq a 
         \\
         [a + \tfrac{c - d}{2}, b + \tfrac{d - c}{2}] & \text{if } a \leq d \text{ and } b < c
         \\
         [a + c - \tfrac{b + d}{2}, \tfrac{b + d}{2}] & \text{if } a \leq d \text{ and } b \geq c
        \end{cases}
     \end{equation*}
\end{lemma}

\begin{proof}
    The symmetry on abstract elements is immediate. In the following, we omit some of $\sharp$ to improve readability. 
    \begin{align*}
        \text{nmin}^\sharp([a, b], [c, d]) 
        &= \frac{1}{2}
            \begin{pmatrix}
                1 & -1 & -1 & -1 
            \end{pmatrix}
            R^\sharp \left(
            \begin{pmatrix}
                1 & 1 \\
                -1 & -1 \\
                1 & -1 \\
                -1 & 1 
            \end{pmatrix}
            \begin{pmatrix}
                [a, b] \\
                [c, d]
            \end{pmatrix}
            \right)
        \\
        &= \frac{1}{2}
            \begin{pmatrix}
                1 & -1 & -1 & -1 
            \end{pmatrix}
            R^\sharp \left(
            \begin{pmatrix}
                [a, b] + [c, d] \\
                -[a, b] -[c, d] \\
                [a, b] -[c, d] \\
                -[a, b] + [c, d] 
            \end{pmatrix}
            \right)
        \\
        &= \frac{1}{2}
            \begin{pmatrix}
                1 & -1 & -1 & -1 
            \end{pmatrix}
            R^\sharp \left(
            \begin{pmatrix}
                [a + c, b + d] \\
                [- b - d , - a - c] \\
                [a - d , b - c] \\
                [c - b, d - a] 
            \end{pmatrix}
            \right)
        \\
        &= \frac{1}{2}
            \begin{pmatrix}
                1 & -1 & -1 & -1 
            \end{pmatrix}
            \begin{pmatrix}
                [R(a + c), R(b + d)] \\
                [R(- b - d), R(- a - c)] \\
                [R(a - d), R(b - c)] \\
                [R(c - b), R(d - a)] 
            \end{pmatrix}
        \\
        &= \frac{1}{2}
            (
                [R(a + c), R(b + d)] 
                - [R(- b - d), R(- a - c)]
                \\
                & \quad - [R(a - d), R(b - c)]
                - [R(c - b), R(d - a)] 
            )
        \\
        &= \frac{1}{2}
        (
            [R(a + c), R(b + d)] 
            + [- R(- a - c), - R(- b - d)]
            \\
            & \quad + [- R(b - c), - R(a - d)]
            + [- R(d - a), - R(c - b)] 
        )
        \\
        &= \frac{1}{2}
        (
            [R(a + c) - R(- a - c), R(b + d) - R(- b - d)] 
            \\
            & \quad + [- R(b - c) - R(d - a), - R(a - d) - R(c - b)]
        )
    \end{align*}
    Claim: $R(a + c) - R(- a - c) = a + c$. If $a + c > 0$ then $- a - c < 0$ thus the claim in this case. Indeed: If $a + c \leq 0$ then $ - a - c \geq 0$ thus $R(a + c) - R(- a - c) = - R(- a - c) = -(-a -c) = a + c$. Similarly $R(b + d) - R(- b - d) = b + d$. 

    So the expression simplifies to
    \begin{align*}
        \text{nmin}^\sharp([a, b], [c, d]) 
        &= \frac{1}{2}
        (
            [a + c, b + d] 
            + [- R(b - c) - R(d - a), - R(a - d) - R(c - b)]
        )
    \end{align*}

    We proceed by case distinction:

    Case 1: $b - c \leq 0$: Then $a \leq b \leq c \leq d$: 
    \begin{align*}
            \text{nmin}^\sharp([a, b], [c, d]) 
            &= \frac{1}{2}
            (
                [a + c, b + d] 
                + [a - d, b - c]
            ) 
            \\
            &= \frac{1}{2} ( [a + c + a - d, b + d + b - c] )
            \\
            &= [a + \tfrac{c - d}{2}, b + \tfrac{d - c}{2}]
        \end{align*}

    Case 2: $a - d \geq 0$: Then $c \leq d \leq a \leq b$. By symmetry of $\text{nmin}$ equivalent to Case 1. Hence
    \begin{align*}
        \text{nmin}^\sharp([a, b], [c, d]) 
        &= [c + \tfrac{a - b}{2}, d + \tfrac{b - a}{2}]. 
    \end{align*}
    
    Case 3: $a - d < 0$ and $b - c > 0$: 
    \begin{align*}
        \text{nmin}^\sharp([a, b], [c, d]) 
        &= \frac{1}{2}
        (
            [a + c, b + d] 
            + [c - b - d + a, 0]
        )
        \\
        &= \frac{1}{2} ( [a + c + c - b - d + a, b + d] )
        \\
        &= [a + c - \tfrac{b+d}{2}, \tfrac{b+d}{2}]
    \end{align*}

    Thus we have 
    \begin{equation*}
        \text{nmin}^\sharp([a, b], [c, d]) = 
        \begin{cases}
            [a + \tfrac{c - d}{2}, b + \tfrac{d - c}{2}] & \text{if } b \leq c
            \\
            [c + \tfrac{a - b}{2}, d + \tfrac{b - a}{2}] & \text{if } d \leq a 
            \\
            [a + c - \tfrac{b + d}{2}, \tfrac{b + d}{2}] & \text{if } a < d \text{ and } b > c
        \end{cases}
     \end{equation*}
\end{proof}

\begin{definition}[neighboring grid points]
    Let $G$ be as above. We define the set of \emph{neighboring grid points} of $x \in \Gamma$ by
    \begin{equation*}
        \mathcal{N}(x) := \{g \in G \mid g \in ||x - g|| \leq \tfrac{1}{M}\} \setminus \{x\}. 
    \end{equation*}
    For $U \subset \mathbb{R}^m$, we define $\mathcal{N}(U) := \{ \mathcal{N}(x) \mid x \in U \} \setminus U$. 
\end{definition}

\begin{definition}[local bump] \label{def:app:local_bump}
    Let $M \in \mathbb{N}$, $G := (\frac{\mathbb{Z}}{M})^m$, $\ell = 2^{\lceil \log_2 2m \rceil+1}$ and let $c = \{\tfrac{i_1^l}{M}, \tfrac{i_1^u}{M}\} \times \cdots \times \{\tfrac{i_m^l}{M}, \tfrac{i_m^u}{M}\} \subseteq G$. 
    We define a ReLU neural network $\phi_c \colon \mathbb{R}^m \to [0, 1]$ w.r.t. the grid $G$ by 
    \begin{equation*}
        \phi_c(x) 
        := R \left(\text{nmin}_{2m}
            \bigcup_{1 \leq k \leq m} 
            \left\{ 
                R_{[*,1]}(M \ell (x_k - \tfrac{i_k^l}{M}) + 1), 
                R_{[*,1]}(M \ell (\tfrac{i_k^u}{M} - x_k) + 1) 
            \right\} 
        \right)
    \end{equation*}
\end{definition}

\begin{lemma} \label{lem:app:local_bump}
    It holds:
    \begin{equation*}
        \phi_c(x) 
        := 
        \begin{cases}
            0 & \text{if } x \notin \text{conv}(\mathcal{N}(c))
            \\
            1 & \text{if } x \in \text{conv}(c)
            \\
            \min \left( 0, \bigcup_{k=1}^m 
            \{ M \ell (x_k - \tfrac{i_k^l}{M}) + 1 \}
            \cup 
            \{ M \ell (\tfrac{i_k^u}{M} - x_k) + 1 \} \right)
            & \text{otherwise}. 
        \end{cases}    
    \end{equation*}
\end{lemma}

\begin{proof} 
    By case distinction:
    \begin{itemize}
        \item Case $x \notin \mathcal{N}(c)$. Then there exists $k$, such that
        either $x_k < \tfrac{i_k^l-1}{M}$ or $x_k > \tfrac{i_k^u+1}{M}$. Then $M
        \ell (x_k - \tfrac{i_k^l}{M}) + 1$  or $M \ell (\tfrac{i_k^u}{M} - x_k)
        + 1$ is less or equal to 0. Hence
        \begin{equation*}
            \phi_c(x) = 0. 
        \end{equation*}
        \item Case $x \in \text{conv}(c)$. Then for all $k$ holds $\tfrac{i_k^l}{M} \leq x_k \leq \tfrac{i_k^u}{M}$. Thus $M \ell (x_k - \tfrac{i_k^l}{M}) + 1 \geq 1$ and $M \ell (\tfrac{i_k^u}{M} - x_k) + 1 \geq 1$ for all k Hence
        \begin{equation*}
            \phi_c(x) = 1. 
        \end{equation*}
        where $\alpha \geq 1$. 
        \item Case otherwise: For all $x$ exists a $k$ such that $ M \ell (x_k - \tfrac{i_k^l}{M}) + 1$ or $M \ell (\tfrac{i_k^u}{M} - x_k) + 1$ is smaller or equal to all other arguments of the function $\min$ and smaller or equal to $1$. If the smallest element is smaller than 0, then $\phi_c(x)$ will evaluate to 0, otherwise it will evaluate to $ M \ell (x_k - \tfrac{i_k^l}{M}) + 1$ or $M \ell (\tfrac{i_k^u}{M} - x_k) + 1$.  
        Thus we can just drop $R$ and $R_{[*,1]}$ from the equations and take the minimum also over 0: 
        \begin{align*}
            \phi_c(x) 
            &= R \left(\min 
            \bigcup_{k=1}^m 
                \left\{ 
                    R_{[*,1]}(M \ell (x_k - \tfrac{i_k^l}{M}) + 1), 
                    R_{[*,1]}(M \ell (\tfrac{i_k^u}{M} - x_k) + 1) 
                \right\} 
            \right)
            \\
            &= \min \left(0, 
            \bigcup_{k=1}^m 
                \{ (M \ell (x_k - \tfrac{i_k^l}{M}) + 1) \}
                \cup
                \{ (M \ell (\tfrac{i_k^u}{M} - x_k) + 1) \}
            \right)
            \\
            &= \min \bigcup_{k=0}^m \{M \ell (x_k - \tfrac{i_k^l}{M}) + 1\} \cup \{M \ell (\tfrac{i_k^u}{M} - x_k) + 1 \}
        \end{align*}
    \end{itemize}
\end{proof}

\begin{lemma}
    Let $[u_1, 1], \dots, [u_{N}, 1]$ be abstract
    elements of the Interval Domain $\mathcal{B}$. Then 
    \begin{equation*}
        \text{nmin}^\sharp_{N}([u_1, 1], \dots, [u_N, 1]) = [u_1 + \cdots u_N + 1 - N, 1]. 
    \end{equation*} 
\end{lemma}

\begin{proof}
    By induction. Base case: Let $N = 1$. Then $\text{nmin}_1^\sharp([u_1, 1]) = [u_1, 1]$. Let $N = 2$. Then $\text{nmin}_2^\sharp([u_1, 1], [u_2, 1]) = [u_1 + u_2 - 1, 1]$. 

    Induction hypothesis: The property holds for $N'$ s.t. $0 < N' \leq N-1$.

    Induction step: Then it also holds for $N$:
    \begin{align*}
        \text{nmin}^\sharp_{N}([u_1, 1], \dots, [u_N, 1]) 
        &= \text{nmin}^\sharp(
        \text{nmin}^\sharp_{\lceil N/2 \rceil}([u_1, 1], \dots, [u_{\lceil N/2 \rceil}, 1]), 
        \\ &\qquad \text{nmin}^\sharp_{N - \lceil N/2 \rceil}([u_{\lceil N/2 \rceil +1}, 1], \dots, [u_N, 1]))
        \\
        &= \text{nmin}^\sharp(
            [u_1 + \cdots + u_{\lceil N/2 \rceil} + 1 - \lceil N/2 \rceil, 1], 
        \\ &\qquad [u_{\lceil N/2 \rceil +1} + \cdots u_N + 1 - N + \lceil N/2 \rceil, 1]
            )
        \\
        &\overset{\cref{lem:app:min_abstract}}{=}
        [u_1 + \cdots + u_N + 2 - \lceil N/2 \rceil - N + \lceil N/2 \rceil - 1, 1] 
        \\
        &=
        [u_1 + \cdots + u_N + 1 - N, 1] 
    \end{align*}
\end{proof}

\begin{lemma} \label{lem:app:min_abstract_minimal_elem}
    Let $[a, b], [u, 1] \in
    \mathcal{B}(\mathbb{R}_{\leq 1})$. Then 
    \begin{equation*}
        \text{nmin}^\sharp([a, b], [u, 1]) \subseteq [a + \tfrac{u - 1}{2}, \tfrac{b + 1}{2}]
    \end{equation*}
\end{lemma}

\begin{proof}
    \begin{align*}
        \text{nmin}^\sharp([a, b], [u, 1]) 
        &= \begin{cases}
            [a + \tfrac{u - 1}{2}, b + \tfrac{1 - u}{2}] &\text{if } b \leq u
            \\
            [a + u - \tfrac{b + 1}{2}, \tfrac{b + 1}{2}] &\text{if } b \geq u
        \end{cases}
    \end{align*}
    If $b \leq u$ then $b + \tfrac{1-u}{2} \leq b + \tfrac{1 - b}{2} = \tfrac{b+1}{2}$. If $u \leq b$ then $a + u - \tfrac{b + 1}{2} \geq a + u - \tfrac{u+1}{2} = a + \tfrac{u-1}{2}$. So 
    \begin{equation*}
        \text{nmin}^\sharp([a, b], [u, 1]) \subseteq [a + \tfrac{u - 1}{2}, \tfrac{b + 1}{2}]. 
    \end{equation*}
\end{proof}

\begin{lemma} \label{lem:app:min_estimate}
    Let $N \in \mathbb{N}_{\geq 2}$, let $[u_1, 1], \dots, [u_{N-1}, 1], [u_N, d] \in \mathcal{B}(\mathbb{R})$ s.t. $b \leq 1$ be abstract elements of the Interval Domain $\mathcal{B}$. Furthermore, let $H(x) := \tfrac{1+x}{2}$. Then there exists a $u \in \mathbb{R}$ s.t. 
    \begin{equation*}
        \text{nmin}^\sharp_{N}([u_1, 1], \dots, [u_{N-1}, 1], [u_N, d]) 
        \subseteq 
        [u, H^{\lceil \log_2 N \rceil + 1}(d)]
    \end{equation*}
\end{lemma}

\begin{proof}
    By induction:
    Let $N = 2$: 
    \begin{equation*}
        \text{nmin}^\sharp_2([u_1, 1], [u_2, d]) \overset{\cref{lem:app:min_abstract_minimal_elem}}{=} [a + \tfrac{u_1 - 1}{2}, H(d)]
    \end{equation*}
    Let $N = 3$:
    \begin{align*}
        \text{nmin}^\sharp_3([u_1, 1], [u_2, 1], [u_3, d]) 
        &= \text{nmin}^\sharp(\text{nmin}^\sharp([u_1, 1], [u_2, 1]), [u_3, d])
        \\
        &= \text{nmin}^\sharp([u_1 + u_2 - 1, 1], [u_3, d])
        \\
        &\subseteq [u_3 + \tfrac{u_1 + u_2 - 2}{2}, H(d)]
        \\
        \text{nmin}^\sharp_3([u_1, 1], [a, b], [u_2, 1]) 
        &= \text{nmin}^\sharp_3([u_3, d], [u_1, 1], [u_2, 1]) 
        \\
        &= \text{nmin}^\sharp(\text{nmin}^\sharp([u_3, d], [u_1, 1]), [u_2, 1])
        \\
        &= \text{nmin}^\sharp([u_3 + \tfrac{u_1 - 1}{2}, H(d)], [u_2, 1])
        \\
        &\subseteq [u_3 + \tfrac{u_1 + u_2 - 2}{2}, H^2(d)]
    \end{align*}
    So $\text{nmin}^\sharp_3([u_3, d], [u_1, 1], [u_2, 1])$ is always included in $[u_3 + \tfrac{u_1 + u_2 - 2}{2}, H^2(d)]$. 

    Induction hypothesis: The statement holds for all $2 \leq N' \leq N-1$. 

    Induction step: Then the property holds also for $N$:

    \begin{align*}
        \text{nmin}^\sharp_{N}([u_N, d], [u_1, 1], \dots, [u_{N-1}, 1]) 
        &= \text{nmin}^\sharp(
            \text{nmin}^\sharp_{\lceil N/2 \rceil}([u_N, d], [u_1, 1], \dots, [u_{\lceil N/2 \rceil -1}, 1]), 
            \\
            & \quad \text{nmin}^\sharp_{N - \lceil N/2 \rceil}([u_{\lceil N/2 \rceil}, 1], \dots, [u_{N-1}, 1]))
        \\
        &= \text{nmin}^\sharp(
            [u', H^{\lceil \log_2 \lceil N/2 \rceil \rceil + 1}(d)], 
            [u'', 1]
            )
        \\
        & \subseteq \text{nmin}^\sharp(
            [u', H^{\lceil \log_2 N/2 \rceil + 1}(d)], 
            [u'', 1]
            )
        \\
        &= \text{nmin}^\sharp(
            [u', H^{\lceil \log_2 N - \log_2(2) \rceil + 1}(d)], 
            [u'', 1]
            )
        \\
        &= \text{nmin}^\sharp(
            [u', H^{\lceil \log_2 N - 1 \rceil + 1}(d)], 
            [u'', 1]
            )
        \\
        &= \text{nmin}^\sharp(
            [u', H^{\lceil \log_2 N \rceil}(d)], 
            [u'', 1]
            )
        \\
        &= [u''', H^{\lceil \log_2 N \rceil + 1}(d)]
    \end{align*}
    and similarly for other orderings of the arguments. 
\end{proof}

\begin{lemma} \label{lem:app:H_identity}
    Let $H(x) := \tfrac{1 + x}{2}$. For all $N \in \mathbb{N}_{>0}$, we have that $d \leq 1 - 2^N$ implies $H^{N}(d) \leq 0$. 
\end{lemma}

\begin{proof} 
    By induction. 
    $N = 1$: Then $H(1 - 2) = \tfrac{1 + 1 - 2}{2} = 0$ 

    Induction hypothesis. The statement holds for all $N'$ such that $0 < N' \leq N$. 

    Induction step: $N+1$: $d \leq 1 - 2^N$:
    \begin{align*}
        H^{N+1}(d) \leq H^{N+1}(1 - 2^{N+1}) = H^N(H(1-2^{N+1})) = H^N(\tfrac{1 + 1 - 2^{N+1}}{2}) = H^N(1 - 2^N) \leq 0
    \end{align*}
\end{proof}

\begin{lemma} \label{lem:app:local_bump_abstract}
    For all boxes $B \in \mathcal{B}(\mathbb{\mathbb{R}^m})$, we have
    \begin{equation*}
        \phi_c^\sharp(B) 
        = 
        \begin{cases}
            [1, 1] &\text{if } B \subseteq \text{conv}(c)
            \\
            [0, 0] &\text{if } B \subseteq \Gamma \setminus \text{conv}(\mathcal{N}(c))
        \end{cases}
    \end{equation*}
    Furthermore, $\phi_c^\sharp(B) \subseteq [0,1]$. 
\end{lemma}

\begin{proof}
    Let $\phi_c$ be a local bump and let $B = [a, b] \in
    \mathcal{B}(\mathbb{R}^m)$. Let $[r_k^1, s_k^1], [r_k^2, s_k^2] \in
    \mathcal{B}(\mathbb{R})$ such that $M \ell ([a_k, b_k] - \tfrac{i_k^l}{M}) +
    1 = [r_k^1, s_k^1]$ and $M \ell (\tfrac{i_k^u}{M} - [a_k, b_k]) + 1 = [r_k^2,
    s_k^2]$.
    \begin{itemize}
        \item If $[a, b] \subseteq \text{conv}(c)$: 
        Then $1 \leq r_k^1$ and $1 \leq r_k^2$ for all $k \in \{1,\dots,m\}$. Thus 
        \begin{align*}
            \phi^\sharp_c([a, b]) 
            &= R^\sharp(\text{nmin}^\sharp_{2m} \{R^\sharp_{[*,1]}([r_k^p, s_k^p])\}_{(p, k) \in \{1, 2\} \times \{1, \dots, m\}} )
            \\
            &= R^\sharp(\text{nmin}^\sharp_{2m} \{[1, 1] \}_{(p, k) \in \{1, 2\} \times \{1, \dots, m\}})
            \\
            &= [1, 1]
        \end{align*}
        \item If $[a, b] \subseteq \Gamma \setminus
        \text{conv}(\mathcal{N}(c))$: Then there exists a $(p', k') \in \{1, 2\}
        \times \{1, \dots, m\}$ such that $s_{k'}^{p'} \leq 1 - 2^{\lceil \log_2
        N \rceil + 1}$. Using \cref{lem:app:H_identity} and
        \cref{lem:app:min_estimate}, we now that there exists a $u \in
        \mathbb{R}$ s.t. 
        \begin{align*}
            \phi^\sharp_c([a, b]) 
            &= R^\sharp(\text{nmin}^\sharp_{2m} \{R_{[*,1]}^\sharp([r_k^p, s_k^p]) \}_{(p, k) \in \{1, 2\} \times \{1, \dots, m\}})
            \\
            &= R^\sharp(\text{nmin}^\sharp_{2m} \{[R_{[*,1]}(r_k^p), R_{[*,1]}(s_k^p)] \}_{(p, k) \in \{1, 2\} \times \{1, \dots, m\}})
            \\
            &\subseteq R^\sharp(\text{nmin}^\sharp_{2m} \{[R_{[*,1]}(r_k^p), 1] \}_{(p, k) \neq (p', k')} \cup \{ [r^{p'}_{k'}, s^{p'}_{k'}] \})
            \\
            &\subseteq R^\sharp([u, 0])
            \\
            &= [0, 0]
        \end{align*}
    \end{itemize}
    For any $[a, b] \in \mathcal{B}(\Gamma)$ we have $\phi_c^\sharp([a, b])
    \subseteq [0,1]$ by construction. 
\end{proof}

\begin{lemma} \label{lem:app:Rn_slice}
    Let $\Gamma \subset \mathbb{R}^m$ be a closed box and let $f \colon \Gamma
    \to \mathbb{R}$ be continuous.
    For all $\delta > 0$ exists a set of ReLU networks $\{n_k\}_{0 \leq k \leq
    N-1}$ of size $N \in \mathbb{N}$ approximating the $N$-slicing of $f$,
    $\{f_k\}_{0 \leq k \leq N-1}$ ($\xi_k$ as in \cref{def:app:const_N_slicing})
    such that for all boxes $B \in \mathcal{B}(\Gamma)$
    \begin{equation*}
        n_k^\sharp(B) =
        \begin{cases}
            [0, 0] 
            &\text{if } f(B) \leq \xi_k - \tfrac{\delta}{2}
            \\
            [1, 1] 
            &\text{if } f(B) \geq \xi_{k+1} + \tfrac{\delta}{2}. 
        \end{cases}
    \end{equation*}
    and $n_k^\sharp(B) \subseteq [0, 1]$. 
\end{lemma}

\begin{proof}
    Let $N \in \mathbb{N}$ such that $N \geq 2 \tfrac{\xi_{\max} -
    \xi_{\min}}{\delta}$ where $\xi_{\min} := \min f(\Gamma)$ and $\xi_{\max} :=
    \max f(\Gamma)$.
    For simplicity we assume $\Gamma = [0,1]^m$. Using the Heine-Cantor theorem,
    we get that $f$ is uniformly continuous, thus there exists a $\delta' > 0$
    such that $\forall x, y \in \Gamma . ||y-x||_\infty < \delta' \Rightarrow
    ||f(y) - f(x)|| < \tfrac{\delta}{2}$. 
    Further, let $M \in \mathbb{N}$ such that $M \geq \tfrac{1}{\delta'}$ and
    let $G$ be the grid defined by $G := (\tfrac{\mathbb{Z}}{M})^m \subseteq
    \mathbb{R}^m$. 

    Let $C(B)$ be the set of corner points of the closest hyperrectangle in $G$
    confining $B \in \mathcal{B}(\Gamma)$. We construct the set
    \begin{equation*}
        \Delta_{k} := \{C(B) \mid 
            B \in \mathcal{B}(\Gamma) : 
            f(B) \geq \xi_{k+1} + \tfrac{\delta}{2}\}. 
    \end{equation*}

    We claim that $\{n_k\}_{0 \leq k \leq N-1}$ defined by 
    \begin{equation*}
        n_k(x) := R_{[*,1]} \left(
            \sum_{c \in \Delta_{k}} \phi_c(x) \right)
    \end{equation*}
    satisfies the condition. 
    
    Case 1: Let $B \in \mathcal{B}(\Gamma)$ such that $f(B) \geq
    \xi_{k+1} + \tfrac{\delta}{2}$. Then for all $g \in \mathcal{N}(B)$ holds
    $f_k(g) = \delta_2$. By construction exists a $c' \in \Delta_{k}$ such
    that $B \subseteq \text{conv}(c')$. Using \cref{lem:local_bump_abstract} we get
    \begin{align*}
        n_k^\sharp(B)
        &= R_{[*,1]}^\sharp \left(
            \sum_{c \in \Delta_{k}} \phi_c^\sharp(B) \right)
        = R_{[*,1]}^\sharp \left(
            \phi_{c'}^\sharp(B)
            +
            \sum_{c \in \Delta_{k} \setminus c'} \phi_c^\sharp(B) \right)
        \\
        &= R_{[*,1]}^\sharp \left(
            [1, 1]
            +
            [p_1, p_2] \right)
        = [1,1],
    \end{align*}
    where $[p_1, p_2] \in \mathcal{B}(\mathbb{R}_{\geq 0})$. Indeed, by case
    distinction:

    Case 2: Let $B \in \mathcal{B}(\Gamma)$ such that $f(B) \leq
    \xi_k - \tfrac{\delta}{2}$. Then for all $g \in \mathcal{N}(B)$ holds
    $f_k(g) = 0$. Further, $B \cap \text{conv}(\mathcal{N}(c)) = \emptyset$ for all $c
    \in \Delta_{k}$ because $G$ is fine enough. 
    Using \cref{lem:local_bump_abstract} we obtain
    \begin{equation*}
        n_k^\sharp(B) 
        = R_{[*,1]}^\sharp \left(
            \sum_{c \in \Delta_{k}} \phi_c^\sharp(B) \right) 
        = R_{[*,1]}^\sharp ([0,0]) = [0,0].
    \end{equation*}

    By construction we have $n_k^\sharp(B) \subseteq [0, 1]$. 
\end{proof}

\message{^^JLASTPAGE \thepage^^J}

\end{document}